\documentclass{article}

\usepackage[final]{neurips_2022}
\usepackage[utf8]{inputenc} 
\usepackage[T1]{fontenc}    
\usepackage{hyperref}       
\usepackage{url}            
\usepackage{booktabs}       
\usepackage{color}
\usepackage{xcolor}
\usepackage{amsfonts}       
\usepackage{nicefrac}       
\usepackage{microtype}      
\usepackage{appendix}
\usepackage{amsthm}
\usepackage{amsmath}
\usepackage{amssymb}
\usepackage{amsfonts}
\usepackage{enumerate}
\usepackage{algorithm}
\usepackage{algorithmic} 
\usepackage{geometry}
\usepackage{xspace}
\usepackage{mathabx}
\usepackage{graphicx}
\usepackage{tabularx, makecell, booktabs}

\usepackage{dsfont}       

\newtheorem{lemma}{Lemma}

\newtheorem{theorem}{Theorem}

\newtheorem{definition}{Definition}

\newtheorem{question}{Question}

\title{Near-Optimal Regret Bounds for Multi-batch Reinforcement Learning}

\author{
Zihan Zhang\\
Tsinghua University \\
\texttt{zihan-zh17@mails.tsinghua.edu.cn} \\
\And
Yuhang Jiang\\
Tsinghua University \\
\texttt{jiangyh19@mails.tsinghua.edu.cn} \\
\AND
Yuan Zhou\\ 
Tsinghua University \\
\texttt{yuan-zhou@tsinghua.edu.cn}
\And
Xiangyang Ji \\
Tsinghua University \\
\texttt{xyji@tsinghua.edu.cn}
}

\begin{document}

\maketitle

\begin{abstract}
	In this paper, 
		we study the episodic reinforcement learning (RL) problem modeled by finite-horizon Markov Decision Processes (MDPs) with constraint on the number of batches. The multi-batch reinforcement learning framework, where the agent is required to provide a time schedule to update policy before everything, which is particularly suitable for the scenarios where the agent suffers extensively from changing the policy adaptively. Given a finite-horizon MDP with $S$ states, $A$ actions and planning horizon $H$, we design a computational efficient algorithm to achieve near-optimal regret of $\tilde{O}(\sqrt{SAH^3K\ln(1/\delta)})$\footnote{$\tilde{O}(\cdot)$ hides logarithmic terms of $(S,A,H,K)$} in $K$ episodes using $O\left(H+\log_2\log_2(K) \right)$ batches with confidence parameter $\delta$. 
		To our best of knowledge, it is the first $\tilde{O}(\sqrt{SAH^3K})$ regret bound with $O(H+\log_2\log_2(K))$ batch complexity. Meanwhile, we show that to achieve $\tilde{O}(\mathrm{poly}(S,A,H)\sqrt{K})$ regret, the number of batches is at least $\Omega\left(H/\log_A(K)+ \log_2\log_2(K) \right)$, which matches our upper bound up to logarithmic terms.
		
			Our technical contribution are two-fold: 1) a near-optimal design scheme to explore over the unlearned states; 2) an computational efficient algorithm to explore certain directions with an approximated transition model.
\end{abstract}

\section{Introduction}

	In reinforcement learning (RL), the learning agent interacts with the environment to maximize the total reward by making sequential decisions. The agent typically has to achieve two seemingly very different goals: to try as many actions and reach as many states as possible so as to learn more information about the environment (a.k.a.~\emph{exploration}) and to follow the policy that collects the high rewards according to the learned information (a.k.a.~\emph{exploitation}). To address this exploration-exploitation dilemma and achieve the near-optimal regret bounds, the agent usually needs to adjust his/her strategies \emph{adaptively} based on the historical trajectories and make frequent policy changes \citep{azar2017minimax,zanette2019tighter,zhang2020almost}.
	
	On the other hand, however, too much adaptivity requirement usually leads to lower level of parallelism, impeding the large-scale deployment of the RL algorithms (which is often in a distributed manner). Frequent policy updates also suffer the cost of re-deploying policies in many practical applications.  
	For example, in medical domains, it often requires complete discussion among many experts to change the treatment plans, which is not affordable in terms of both time and monetary cost \citep{lei2012smart,almirall2012designing,almirall2014introduction}; in RL for hardware placement \citep{mirhoseini2017device}, rewriting the program into the hardware for too many times is strongly discouraged.
	Similar challenges also arise in applying RL to personalized recommendation system \citep{yu2019convergent} and database optimization \citep{krishnan2018learning}.
	
	In such cases, the learning agent should minimize the number of policy switches while keeping the regret affordable.  \cite{bai2019provably} first proposed the provably efficient RL algorithms with low switching costs under the $Q$-learning algorithmic framework together with the lazy update techniques. However, their method needs to actively monitor the data in real time to determine whether a policy change is to be initiated. In other words, although the number of policy switches by \citep{bai2019provably} is low, the (usually long) time periods when the same policy is used still cannot be parallelized due to the policy-change trigger in their algorithms which is intrinsically sequential.
	
	In order to address this problem, we propose and study under the framework of \emph{multi-batch RL}, where the learning agent has to determine the number of batches and length of each batch before the learning process starts,\footnote{In contrast, \cite{bai2019provably} can update the policy at any time.} and uses as few batches as possible to achieve a low regret. Multi-batch RL algorithms can be easily deployed in a distributed fashion as the episodes during the same batch can be easily and fully parallelized. The idea of batch learning is also being widely practiced.  For example, in medical trials, the medical center usually collects the data during a fixed time period among a batch of patients and then designs the experiment for the next phase based on the learned information in previous phases \citep{lei2012smart,almirall2012designing,almirall2014introduction}.

	Formally, we define multi-batch RL and \emph{batch complexity} as below.
	\begin{definition}[Multi-Batch RL with complexity $M$]\label{def_batch_RL}
		The agent determines a group of lengths $\{t_m\}_{m=1}^M$ such that $\sum_{m=1}^M t_m = K$ before the learning process starts. For $m=1,2,\ldots,M$, the agent sets a policy $\pi^m$
		and then follows $\pi^m$ for $t_m$ episodes.
	\end{definition}

	We highlight that an upper bound for batch complexity implies the same upper bound for global switching cost, since each policy switch means a new batch.  It is also worth noting that the proposed batch RL framework is fully parallelizable during each batch for the applications where dataset comes in batch (e.g., clinical trial). Like other RL settings,  we have the natural and interesting question:
	\begin{question}\label{question:main}
		Is it possible to achieve near optimal batch complexity, while keeping the  regret $\tilde{O}(\sqrt{SAH^3K})$.
	\end{question}
	We provide a positive answer for Question~\ref{question:main}, which we state as below.
	\begin{theorem}\label{thm:main} Let\footnote{Throughout the paper we use $\iota$ to denote $\ln(2/\delta)$.} $\iota = \ln(2/\delta)$.
		For any episodic MDP,  with probability $1-\delta$, under Algorithm~\ref{alg:main} the regret in $T$ episodes is bounded by $$\mathrm{Regret}(T)  \leq \tilde{O}\left(     \sqrt{SAH^3K\iota^2} +  S^{\frac{15}{4}}A^{\frac{9}{8}}H^{\frac{17}{8}}\iota^{\frac{5}{8}}K^{\frac{3}{8}}+   S^{\frac{19}{4}}A^{\frac{13}{4}}H^{\frac{33}{4}}\iota K^{\frac{1}{4}}+ S^{\frac{11}{2}}A^{\frac{9}{2}}H^{\frac{17}{2}}\iota \right),$$ and the batch complexity is bounded by $O(H +\log_2\log_2(K))$. Moreover, the computational cost of Algorithm~\ref{alg:main} is $\tilde{O}(S^4AHK^3+S^3A^2H^2K^3)$.
	\end{theorem}

On the other hand, we show a lower bound of batch complexity as below.
	\begin{theorem}\label{thm:lb}
		For any algorithm with $O(\mathrm{poly}(S,A,H)\sqrt{K})$ regret bound, the batch complexity is at least $\Omega(H/\log_A(K)+\log_2\log_2(K))$.
	\end{theorem}
	Compared to the lower bound of $\Omega(\log_2\log_2(K))  $ in \citep{gao2019batched}  for multi-armed bandit problem,  additional $\Omega(H/\log_A(K))$ batches are required to explore the structure of the MDP.

Due to space limitation, we defer the full proofs of Theorem~\ref{thm:main} and Theorem~\ref{thm:lb} to Appendix~\ref{sec:proof} and  Appendix~\ref{app:lb} respectively.

	\paragraph{Our contribution.} We propose the framework of multi-batch RL, and first achieve $O(H+\log_2\log_2(K))$  sample complexity bound with the near-optimal $\tilde{O}(\sqrt{SAH^3K\iota})$ regret bound with an efficient algorithm. We also prove that for any algorithm with $O(\mathrm{poly}(S,A,H)\sqrt{K})$ regret, the global switching cost is at least $\Omega(H/\log_A(K)+\log_2\log_2(K))$, which implies a nearly matching  lower bound of $\Omega(H/\log_2(K)+\log_2\log_2(K))$ for the batch complexity.
	We also note that the $O(H+\log_2\log_2(K))$  batch complexity implies an $O(H+\log_2\log_2(K))$ bound for the global switching cost, which is also a near optimal upper bound.

\section{Related Works}

	\paragraph{Bandit algorithms with limited adaptivity.} 
	Bandit problem with low switching cost is widely studied in past decades \citep{cesa2013online,perchet2016batched,gao2019batched, simchi2019phase}. \cite{cesa2013online}  showed an $\tilde{\Theta}(K^{\frac{2}{3}})$ regret bound under adaptive adversaries and bounded memories. \cite{perchet2016batched} proved a regret bound of $\tilde{\Theta}(K^{\frac{1}{1-2^{1-M}}})$ for the two-armed bandit problem within $M$ batches, and later
 \cite{gao2019batched} extended their result to the general $A$-armed case.
	Besides the setting of classical multi-armed bandit problem, other settings has also been studied, 
	e.g., multinomial bandit problem \citep{dong2020multinomial} and linear bandit problem \citep{ruan2020linear}.

	\paragraph{Episodic reinforcement learning with low switching cost.}
	For model-based algorithms, by doubling updates, the global switching cost is $O(SAH\log_2(K))$ while keeping the regret $\tilde{O}(\sqrt{SAKH^3})$\cite{azar2017minimax}. 
	For model-free algorithms,
	\cite{bai2019provably} first studied RL with low switching cost. They proposed a  $Q$-learning algorithm with lazy update to achieve $\tilde{O}(\sqrt{SAKH^4})$ regret bound and $O(SAH^3\log(K/A))$ local switching cost.  Recently \cite{zhang2020almost} established a better regret bound of $\tilde{O}(\sqrt{SAKH^3})$ and $O(SAH^2\log(K/A))$ local switching cost. Besides, \cite{gao2021provably} generalized the problem to Linear RL, and established a regret bound of $\tilde{O}(\sqrt{d^3H^4K})$ with $O(dH\log(K))$ global switching cost. Recent work \cite{qiao2022sample} achieved $O(HSA\log_2\log_2(K))$ switching cost and  $\tilde{O}(\mathrm{poly}(S,A,H)\sqrt{K})$ regret with a computational inefficient algorithm.
	
	\paragraph{Regret minimization for reinforcement learning.}
	There is a long line of works devoting to regret minimization for RL problem
	\citep{kakade2003sample,jaksch2010near,bartlett2009regal,dann2019policy,azar2017minimax,jin2018q,zanette2019tighter,zhang2019regret,zhang2020almost,li2020breaking,zhang2021reinforcement}. For tabular setting, near optimal regret bound  of $\tilde{O}(\sqrt{SAH^3T})$ has been established by \citep{azar2017minimax,zanette2019tighter,zhang2020almost} for both model-based and model-free algorithms. 
	However, fewer algorithms focused on the setting of multi-batch RL.

	\section{Preliminaries}\label{sec:pre}
	
	\paragraph{Episodic reinforcement learning.} $M =\left \langle \mathcal{S}, \mathcal{A} ,r ,P ,s_1   \right \rangle$, where $\mathcal{S} \times\mathcal{A}$ is the discrete state-action space, $r = \{r_h(s,a) \}_{(s,a)\in \mathcal{S}\times \mathcal{A},h\in [H] }$ is the known\footnote{This is a common assumption since the uncertainty of reward function is dominated by that of the transition model.} reward function, $P =\{P_{h}(s,a) \}_{(s,a)\in \mathcal{S}\times \mathcal{A},h\in [H] }$ is the unknown transition model and $s_1$ is the fixed initial state\footnote{The more general case, where the agent starts from a fixed initial distribution, could be reduced to our setting by increasing $H$ by $1$ }.  We assume that the reward function $r_{h}(s,a)\in [0,1]$ for any $(h,s,a)$.
	In each episode, the agent starts at $s_1$, then takes actions and transits to the next state step by step, and finally conducts the trajectory $\{(s_h,a_h,s_{h+1})\}_{h=1}^{H}$. The target of the  agent is to maximize the accumulative reward function $\sum_{h=1}^H r_h(s_h,a_h)$.

	A policy $\pi$ can be viewed as a series of mappings $\{ \pi_h \}_{h=1}^{H}$ where $\pi_h : \mathcal{S} \to \Delta^{\mathcal{A}} $ maps $s_h$ to a distribution over the action space at the $h$-th step, where $\pi_h(a|s)$ is the probability taking action $a$ at state $s$ of the $h$-th horizon.

	Given a policy $\pi$, the (optimal) $Q$-function and value function are given by 
	\begin{align}
	& Q_h^{\pi}(s,a)  = \mathbb{E}_{\pi}\left[  \sum_{h' = h}^H r_{h'}(s_{h'},a_{h'})\Big|(s_h,a_h) = (s,a) \right]; \quad \quad \quad  Q_h^*(s,a) = \sup_{\pi\in \Pi}Q_h^{\pi}(s,a) ;\nonumber
	\\ &V_h^{\pi}(s)  = \mathbb{E}_{\pi}\left[ \sum_{h' = h}^H r_{h'}(s_{h'},a_{h'})\Big|s_h= s\right];  \quad \quad \quad \quad \quad \quad\quad \quad V_h^*(s) = \max_a Q_h^*(s,a). \nonumber
	\end{align}

	Let $\pi^{(k)}$ denote the policy in the $k$-th episode. Then
	the regret is given by 
	\begin{align}
	\mathrm{Regret}(K):=\sum_{k=1}^{K}(V^{*}_{1}(s_1)- V^{\pi^{(k)}}_{1}(s_1)).
	\end{align}

	\paragraph{Notations}
	In this paper, 
	we use 
	$\mathbb{E}_{\pi,p}[\cdot]$ ($\mathbb{P}_{\pi,p}[\cdot]$) to denote the expectation (probability) following policy $\pi$  under transition model $p$.  In particular, $\mathbb{E}_{\pi}[\cdot]$($\mathbb{P}_{\pi}[\cdot]$) denotes the expectation (probability) following $\pi$ under the true transition model $P$. We  define the general value function
	\begin{align}
&	W^{\pi}(r',p)  = \mathbb{E}_{\pi,p}\left[  \sum_{h=1}^H r'_h(s_h,a_h) \right].\nonumber
	\end{align}
We use  $\textbf{1}$ to denote the $S$-dimensional vector $[1,1,\ldots,1]^{\top}$ and $\textbf{1}_{h,s,a}$ to denote the reward function  $r'$ such that $r'_{h'}(s',a')=\mathbb{I}[(h,s,a)=(h',s',a')]$. We also define $\{d_h^{\pi}(s,a)\}_{(s,a,h)}$ be the occupancy distribution of $\pi$. That is, $d_h^{\pi}(s,a) = \mathbb{E}_{\pi}[\mathbb{I}[(s_h,a_h)=(s,a)] ]$. $\Delta^{d}$ is used to denote the $d$-dimensional simplex. For two vector $x,y$ with the same dimension, we write $x^{\top}y$ as $xy$ for convenience. For $p\in \Delta^{S}$ and $v\in \mathbb{R}^S$, we define $\mathbb{V}(p,v)=pv^2-(pv)^2$. For $N\geq 1$, we use $[N]$ to denote the set $[1,2,\ldots,N]$.

\section{Technique Overview}\label{sec:tec}
In this section, we first introduce the policy elimination framework, which enjoys the near-optimal batch complexity. Then we summarize the technical challenges to achieve the near-optimal regret bound efficiently under this framework. At last, we introduce our major technical contributions.

\subsection{Policy Elimination Framework}
 Following the methods in multi-batch bandit learning \cite{perchet2016batched,gao2019batched}, we construct our main algorithm using policy elimination. Like most model-based reinforcement learning methods, we maintain a confidence region $\mathcal{P}$ for the transition model, where the true transition model $P\in\mathcal{P}$ with high probability. Before each batch starts, for a policy $\pi$ and a reward function $u$, by extended value iteration (See Algorithm~\ref{alg:evi} in Appendix~\ref{app:cei}), we are able to compute the confidence interval $[L^{\pi}(u,\mathcal{P}), U^{\pi}(u,\mathcal{P})]$ for the value function of $\pi$, where
\begin{align}
U^{\pi}(u,\mathcal{P}):=  \max_{p'\in \mathcal{P}} W^{\pi}(u+\textbf{1}_{z},p'); \quad \quad \quad  L^{\pi}(u,\mathcal{P}): = \min_{p'\in \mathcal{P}} W^{\pi}(u,p').\label{eqdeful}
\end{align}
Here $z$ is a virtual state for the \emph{infrequent} state-action-state triples (See Function $\mathtt{clip}$ in Algorithm~\ref{alg:raw_exploration}). The reason why we give reward $1$ for $z$ in computing the upper confidence bound is to encourage exploration to these \emph{infrequent} state-action-state triples.

By 
policy elimination we get   $
\Pi(r,\mathcal{P}) =	\left\{ \pi \big|  U^{\pi}(r,\mathcal{P}) \geq \sup_{\pi'}L^{\pi'}(r,\mathcal{P})   \right\}$ as the set of survived policies. The next step is to choose a policy  $\pi\in \Pi(r,\mathcal{P})$ and execute $\pi$ in the current batch. 
Defining $\mathcal{P}^m$ to be the confidence region for the transition model after the $m$-th batch and $\mathrm{gap}^{m+1} = \max_{\pi\in \Pi(r,\mathcal{P}^m)}(U^{\pi}(r,\mathcal{P}^m)-L^{\pi}(r,\mathcal{P}^m))$, the regret in the $m+1$-th batch could be  bounded by $t^{m+1}\mathrm{gap}^{m+1}$. 
Therefore, the main task is to design efficient exploration policy to reduce $\mathrm{gap}^m$ for each $1\leq m \leq M$.



\subsection{Technical Challenges }
Following the policy elimination framework above, we have two major challenges to achieve the near-optimal regret bound with an efficient algorithm.

\paragraph{Difficulty in exploration} 
Fix the reward function $r$ and confidence region $\mathcal{P}$. To construct tight confidence interval for every policy $\pi\in \Pi(r,\mathcal{P})$, we need to find a policy $\pi\in \Pi(r,\mathcal{P})$ to collect enough samples for each $(h,s,a)$. To address the problem, \cite{qiao2022sample} proposed an algorithm named APEVE, which learns each $(h,s,a)$ triple independently. More precisely, for each $(h,s,a)\in [H]\times \mathcal{S}\times\mathcal{A}$, the algorithm searches for a policy $\pi^{h,s,a}$  to maximize the probability of visiting $(h,s,a)$ over $\Pi(r,\mathcal{P})$, and then execute $\pi^{h,s,a}$ to collect samples for $(h,s,a)$. However, this algorithm might be inefficient in sampling, since different horizon-state-action triples may match along with the same exploration policy. As shown in \cite{qiao2022sample}, the regret bound might be sub-optimal with this algorithm. Therefore, to achieve the near-optimal regret bound, we need to design a new exploration strategy to utilize the 
correlationship among different horizon-state-action triples.

\paragraph{Difficulty in efficient implementation} Because the policy set $\Pi(r,\mathcal{P})$ might have exponential size, naive enumeration is not applicable to searching for a good exploration policy.
 As a consequence, it requires additional efforts to study the structure of $\Pi(r,\mathcal{P})$. For example, when $r=0$, $\Pi(r,\mathcal{P})$ is the set of all possible policies. In this case, we can use extended value iteration (See Algorithm~\ref{alg:evi}) to find the policy which visits $(h,s,a)$ most frequently. 

\subsection{Key Techniques}

\paragraph{Near-optimal design scheme}
Unlike RL algorithm with limited switching cost, in multi-batch reinforcement learning, the agent can not change the policy adaptively. As a result, we need to design a  policy with proper coverage ratio for all the survived policies. That is, using the data collected following this policy, the length of the confidence interval for any survived policy is bounded by a uniform threshold.

Recall that $d^{\pi}_h(s,a) = \mathbb{E}_{\pi}[\mathbb{I}[(s_h,a_h)=(s,a)]$. 
Using classical regret analysis for tabular RL \citep{azar2013minimax,zanette2019tighter}, for a fixed  policy $\pi$, the length of confidence interval for $\pi$   could be roughly bounded by 
\begin{align}
    \tilde{O}\left( \sum_{s,a,h}d^{\pi}_h(s,a)\sqrt{\frac{\mathrm{Var}_h(s,a)}{N_h(s,a)}}\right) \substack{\leq \\ \mathrm{Cauchy's \, ineq.}} \tilde{O}\left( \sqrt{\sum_{s,a,h}\frac{d_h^{\pi}(s,a)}{N_h(s,a)}}\cdot\sqrt{\sum_{s,a,h}d_h^{\pi}(s,a)\mathrm{Var}_h(s,a)}\right),\label{eq:wxx}
\end{align}
where $\mathrm{Var}_h(s,a)$ is the variance term with respect to $P_{h,s,a}$ and $V*_{h+1}(\cdot)$, 
and $N_h(s,a)\geq 1$ is the count of $(h,s,a)$.

Because $\sum_{s,a,h}d_h^{\pi}(s,a)\mathrm{Var}_h(s,a)$ could be uniformly bounded by $O(H^2)$ using classical analysis, we focus on bounding the term $\sum_{s,a,h}\frac{d_h^{\pi}(s,a)}{N_h(s,a)}$. Suppose the policy for current batch is $\tilde{\pi}$. After this batch, we roughly have that $N_h(s,a)\propto d^{\tilde{\pi}}_h(s,a)$. So it corresponds to find a policy $\tilde{\pi}\in \Pi(r,\mathcal{P})$ to minimize the \emph{worst-case coverage number} $\max_{\pi\in \Pi(r,\mathcal{P})}\sum_{h,s,a}\frac{d^{\pi}_h(s,a)}{d^{\tilde{\pi}}_h(s,a)}$. For this problem, we have the lemma below, and the proof is deferred to Appendix~\ref{app:pflemma1}.
\begin{lemma}\label{lemma:design} Let $d>0$ be an integer. 
Let $\mathcal{X}\subset (\Delta^d)^{m}$. Then there exists a distribution $\mathcal{D}$ over $\mathcal{X}$, such that 
\begin{align}
\max_{x=\{x_i\}_{i=1}^{dm}\in \mathcal{X}}\sum_{i=1}^{dm}\frac{x_i}{y_i} = md,\nonumber
\end{align}
where $y =\{y_i\}_{i=1}^{dm}= \mathbb{E}_{x\sim \mathcal{D}}[x]$. Moreover, if $\mathcal{X}$ has a boundary set $\partial \mathcal{X}$ with finite cardinality, we can find an approximation solution for $\mathcal{D}$ in $\mathrm{poly}(|\partial \mathcal{X}|)$ time.
\end{lemma}
Plugging $\mathcal{X}=\left\{\{ d^{\pi}_h(\cdot,\cdot)\}_{h=1}^H| \pi\in \Pi(r,\mathcal{P}) \right\}$, $d=SA$ and $m=H$ into Lemma~\ref{lemma:design}, there exists a policy $\tilde{\pi}$ being a mixture of policies in $\Pi(r,\mathcal{P})$,  
such that $\max_{\pi\in \Pi(r,\mathcal{P})}\sum_{s,a,h}\frac{d^{\pi}_h(s,a)}{d^{\tilde{\pi}}_h(s,a)}=SAH$. In this way, we can find the desired exploration policy $\tilde{\pi}$ by assuming the knowledge of $\left\{d_h^{\pi}(\cdot,\cdot)\right\}_{h=1}^H$  for all $\pi\in \Pi(r,\mathcal{P})$.

Given the design scheme above, it remains two problems, for which we present solutions below: 1)\emph{ $\left\{d_h^{\pi}(\cdot,\cdot)\right\}_{h=1}^H$ is unknown;} 2) \emph {even assuming $\left\{d_h^{\pi}(\cdot,\cdot)\right\}_{h=1}^H$ is known, it is hard to find $\tilde{\pi}$ since the cardinality of $\left\{\{ d^{\pi}_h(\cdot,\cdot)\}_{h=1}^H| \pi\in \Pi(r,\mathcal{P}) \right\}$ might be exponential in $SH$.}

\paragraph{Constructing tight confidence region} To estimate  $\left\{d_h^{\pi}(\cdot,\cdot)\right\}_{h=1}^H$, 
we consider to construct a tight confidence region for the transition model to estimate the occupancy distribution up to a constant ratio. 
\begin{definition}\label{def:tight}
We say a confidence transition region $ \mathcal{P}=\otimes_{h,s,a}\mathcal{P}_{h,s,a}$  is \emph{tight} with respect to $p'$ iff (\romannumeral1)$p'\in \mathcal{P}$; (\romannumeral2)
 $e^{-\frac{1}{H}}p'_{h,s,a,s'}\leq p_{h,s,a,s'}\leq e^{\frac{1}{H}}p'_{h,s,a,s'}$
 for any $(h,s,a,s')$ and any $p_{h,s,a}\in \mathcal{P}_{h,s,a}$; (\romannumeral3) $\mathcal{P}_{h,s,a}$ has the form $\mathcal{P}_{h,s,a}=\{p\in\Delta^{S} |   a_i^{\top}p\leq b_i,i=1,2,...,m  \}$ where $m\leq \mathrm{poly}(SM)$.
 \end{definition}
 In model-based reinforcement learning, these conditions are natural and it is easy to construct a \emph{tight} confidence region with acceptable error.
 
 Once we have a confidence region which is $\emph{tight}$ w.r.t. the true transition model $P$, for any policy $\pi$ and $(h,s,a)$, we can estimate the expected visit count  $W^{\pi}(\textbf{1}_{h,s,a})$ by $W^{\pi}(\textbf{1}_{h,s,a},p)$ for any $p\in \mathcal{P}$ because
\begin{align}
  e^{-1} W^{\pi}(\textbf{1}_{h,s,a},p)\leq  W^{\pi}(\textbf{1}_{h,s,a})=d^{\pi}_h(s,a)\leq  e W^{\pi}(\textbf{1}_{h,s,a},p).\nonumber 
\end{align}
With $ W^{\pi}(\textbf{1}_{h,s,a},p)$ as approximation of $d_h^{\pi}(s,a)$, we can continue the analysis above by paying a constant factor.

To learn such a confidence region, by Bennet's inequality (Lemma~\ref{lemma:bennet}), it suffices to visit $(h,s,a,s')$\footnote{A tuple $(h,s,a,s')$ is visited means $(s_h,a_h,s_{h+1})=(s,a,s')$.} for $C_1H^2\iota$ for each $(h,s,a,s')$, where $C_1$ is an universal constant.
By this idea, we try to visit each $(h,s,a,s')$ as much as possible.
In the meantime, it is very possible that some $(h,s,a,s')$ tuples are extremely hard to visit. Fortunately, with proper exploration scheme, we can show that the maximal probability to visit such tuples is well-bounded, so that these tuples could be ignored by suffering regret $O(\sqrt{T})$.

\paragraph{Computational efficient design scheme}
Assume the confidence region $\mathcal{P}$ is \emph{tight} w.r.t. $P$. 
We invoke reward-zero exploration to learn a sub-optimal solution for the problem $\min_{\tilde{\pi}\in \Pi(r,\mathcal{P})}\max_{\pi\in \Pi(r,\mathcal{P})}\sum_{h,s,a}\frac{d^{\pi}_h(s,a)}{d^{\tilde{\pi}}_h(s,a)}$.
Let $p\in \mathcal{P}$ be fixed and define $\tilde{d}^{\pi}_h(s,a)=W^{\pi}(\textbf{1}_{h,s,a},p)$ be the approximation for
$d^{\pi}_h(s,a)$. We define $\tilde{\pi}^i = \arg\max_{\pi\in \Pi(r,\mathcal{P})}W^{\pi}(r^i,p)$ for $1\leq i \leq k = K^3$, where $r^i_{h}(s,a) = \min\left\{ \frac{1}{\sum_{j=1}^{i-1}\tilde{d}_h^{\tilde{\pi}^j}(s,a)},1\right\}$.
Let $\tilde{\pi}$ be the mixture of $\{\tilde{\pi}^i\}_{i=1}^{k}$.  For any policy $\pi$,  we have that
\begin{align}
\sum_{s,a,h}d^{\pi}_h(s,a)\cdot \min \left\{\frac{1}{d^{\tilde{\pi}}_h(s,a)},k \right\} & \leq O \left(\sum_{s,a,h}\tilde{d}^{\pi}_h(s,a)\cdot \min \left\{\frac{1}{\tilde{d}^{\tilde{\pi}}_h(s,a)},k \right\}\right)\label{eq:cr1}
\\ & \leq O\left(\sum_{i=1}^k W^{\pi}(r^{i},p)\right)\label{eq:cr2}
\\ & \leq O\left(\sum_{i=1}^k W^{\tilde{\pi}^i}(r^i,p) \right)\label{eq:cr3}
\\ &  \leq O\left(\sum_{s,a,h} \sum_{i=1}^k d^{\tilde{\pi}^i}_h(s,a)\cdot \min\left\{ \frac{1}{\sum_{j=1}^{i-1}d^{\tilde{\pi}^j}_h(s,a) },1\right\}          \right)\nonumber
\\ & \leq  O\left(\sum_{s,a,h} \sum_{i=1}^k \log\left( \frac{ \max\{ \sum_{j=1}^{i}d^{\tilde{\pi}^j}_h(s,a) ,1\}}{     \max\{ \sum_{j=1}^{i-1}d^{\tilde{\pi}^j}_h(s,a) ,1\}   } \right)        \right)\nonumber
\\ & \leq 
O(SAH\log(k)).\label{eq:wxxx2}
\end{align}
Here \eqref{eq:cr1} holds by the tightness of $\mathcal{P}$, \eqref{eq:cr2} holds by the fact that $r^i_h(s,a)\geq r^{k+1}_h(s,a) = \min\left\{ \frac{1}{\sum_{j=1}^{k}\tilde{d}^{\tilde{\pi}^j}_h(s,a)},1 \right\} = \frac{1}{k}\min\left\{ \frac{1}{\tilde{d}^{\tilde{\pi}}_h(s,a)},k \right\}$ for any $(h,s,a)$, and \eqref{eq:cr3} holds by the optimality of $\tilde{\pi}^i$ for $1\leq i \leq k$.
With \eqref{eq:wxxx2} in hand, $\max_{\pi\in \Pi(r,\mathcal{P})}\sum_{h,s,a}\frac{d^{\pi}_h(s,a)}{d^{\tilde{\pi}}_h(s,a)}$ is roughly bounded by $O(SAH\log(K))$\footnote{We remark the there is still a gap between $\max_{\pi\in \Pi(r,\mathcal{P})}\sum_{h,s,a}\frac{d^{\pi}_h(s,a)}{d^{\tilde{\pi}}_h(s,a)}$ and $\sum_{s,a,h}d^{\pi}_h(s,a)\cdot \min \left\{\frac{1}{d^{\tilde{\pi}}_h(s,a)},K^3 \right\}$.  Actually \eqref{eq:wxxx2} is sufficient for further regret analysis. 
}, which nearly matches the best \emph{worst-case coverage number} number of $SAH$.

\paragraph{Computational efficient constrained exploration} Let $u,u'$ be two reward functions and $\mathcal{P}$ be a set of transition models. As stated before, for general  $\Pi(u,\mathcal{P})$ , it might be non-trivial to solve the problem $\tilde{\pi} = \arg\max_{\pi\in \Pi(u,\mathcal{P})}W^{\pi}(u',p)$ for  fixed $p\in \mathcal{P}$. As a trade-off, we turn to find some policy $\tilde{\pi}\in \Pi(u,\mathcal{P})$ such that $W^{\tilde{\pi}}(u',p)\geq c\max_{\pi\in \Pi(u,\mathcal{P})}W^{\pi}(u',p)$, where $c>0$ is some universal constant. The problem turns out to be a RL problem with a soft constraint. For general $\Pi(u,\mathcal{P})$, the problem might be hard to solve. Fortunately, on the benefit of the \emph{tight} property of $\mathcal{P}$, we can find such $\tilde{\pi}$ efficiently.

\section{Algorithms}\label{sec:alg}
In this section we present our algorithms.
The main algorithm (Algorithm~\ref{alg:main}) consists of three stages. 

In the first two stages, we conduct naive exploration to identify the tuples which are hard to visit, which we called \emph{infrequent} tuples. In particular, the length of the second stage is slightly larger than that of the first stage, where we use the dataset in the first stage to reduce the regret in the second stage. In this way, we can bound the regret in the first two stages by $\tilde{O}(\sqrt{SAH^3K})$, while the probability of visiting the \emph{infrequent} tuples is small enough.

After ignoring the \emph{infrequent} tuples, we could obtain a \emph{tight} confidence region. Given the \emph{tight} confidence region, we compute the confidence region for each policy and conduct policy elimination in the third stage. 
The first and second stages contains $O(H)$ batches, and the third stage contains $O(\log_2\log_2(K))$ batches. So the batch complexity of Algorithm~\ref{alg:main} is $O(H+\log_2\log_2(K))$.
Below we describe $\mathtt{Raw\,Exploration}$ (Algorithm~\ref{alg:raw_exploration}) and $\mathtt{Policy\,Elimination}$ (Algorithm \ref{alg:policy_elimination}) in detail.

	\begin{algorithm}[tb]
	\caption{$\mathtt{Main \,\,Algorithm}$ }
	\begin{algorithmic}[1]\label{alg:main}
	\STATE{\textbf{Input:} state-action space $\mathcal{S}\times\mathcal{A}$, number of episodes $K$, confidence parameter $\delta$;}
	\STATE{\textbf{Initialize:} $\iota\leftarrow \ln(2/\delta)$, $k_1 \leftarrow 144\sqrt{SAKH\iota}$, $k_2 \leftarrow 288S^3A^2H^4\sqrt{K\iota} $;}
		\STATE{ $\{ \mathcal{D}_1 \}\leftarrow\mathtt{Raw \,\, Exploration}(0,\emptyset,k_1)$;}\label{line:mp}
		\STATE{$\{\mathcal{D}_2\}\leftarrow\mathtt{Raw \,\, Exploration}(r,\mathcal{D}_1,k_2)$;}
		\STATE{$\mathtt{Policy \,Elimination}( \mathcal{D}_2 , K-Hk_1-Hk_2)$.}
	\end{algorithmic}
\end{algorithm}

\begin{algorithm}[tb]
	\caption{ $\mathtt{Raw\,Exploration}(u, \mathcal{D},k)$}
	\begin{algorithmic}[1]\label{alg:raw_exploration}
	\STATE{\textbf{Input}: reward function $u$, dataset $\mathcal{D}$, length $k$;}
		\STATE{\textbf{Initialize:}  
		$C_1\leftarrow 200$;}
		\FOR{$h=1,2,\ldots,H$}
			\STATE{$\mathcal{P}\leftarrow \mathtt{CR}(\mathcal{D})$;}
		\FOR{$(s,a)\in \mathcal{S}\times\mathcal{A}$}
		\STATE{$\pi^{h,s,a}\leftarrow \mathtt{Policy\,Search}(u,\textbf{1}_{h,s,a},\mathcal{P})$;\label{line:pss}}
		\ENDFOR
		\STATE{$p\leftarrow $ arbitrary element in $\mathcal{P}$;}
		\STATE{$\{\tilde{\pi}^{h},p\}\leftarrow \mathtt{Sum}\left( \left\{\frac{1}{SA},\pi^{h,s,a}, p \right\}_{(h,s,a)} \right)$;\label{line:pih} }
		\STATE{$\pi^{h}$ be the policy which is the same as $\tilde{\pi}^{h}$ in the first $h-1$ steps, and be the uniformly random policy in the left $H-h+1$ steps;}
		\STATE{Execute $\pi^{h}$ for $k$ episodes, and collect the samples as $\mathcal{D}_h$;}
		\STATE{$\mathcal{D}\leftarrow \mathcal{D}\cup \mathcal{D}_{h}$;}
		\ENDFOR
		\STATE{\textbf{return:} $\{\mathcal{D}\}$;}
				\STATE{\vspace{1ex}
				\textbf{Function}: $\mathtt{CR}(\mathcal{D})$:
				}
			\STATE{\quad $N_h(s,a,s')\leftarrow $ count of $(h,s,a,s')$ in $\mathcal{D}$, for all $(s,a,s')$;\label{line_n}}
		\STATE{\quad $N_h(s,a)\leftarrow \max\{\sum_{s'}N_h(s,a,s'),1\}$ for all $(s,a)$;}
		\STATE{\quad $\hat{p}_{h,s,a,s'}\leftarrow \frac{N_h(s,a,s')}{N_h(s,a)}$, $\forall (h,s,a,s')$;}
		\STATE{\quad $\mathcal{W}\leftarrow \{(h,s,a,s'):\, N_h(s,a,s')\geq C_1H^2\iota\}$;}
		\STATE{\quad $\tilde{\mathcal{P}}_{h,s,a}\leftarrow \left\{p\in\Delta^{S}| \left|p_{s'}-\hat{p}_{h,s,a,s'}\right|\leq \alpha(N_{h}(s,a),N_h(s,a,s')) , \forall s'\in \mathcal{S}\right\}$, $\forall (h,s,a)$;\label{line:alpha}}
		\STATE{\quad $\mathcal{P}_{h,s,a}\leftarrow \{ \mathtt{clip}(p,\mathcal{W}): p\in \tilde{\mathcal{P}}_{h,s,a}\}$, $\forall (h,s,a)$;}
		\STATE{\quad \textbf{Return}: $\otimes_{h,s,a}\mathcal{P}_{h,s,a}$.}
		\STATE{\vspace{1ex}
		\textbf{Function}: $\mathtt{clip}(p,\mathcal{W})$}
		\STATE{\quad $p'_{h,s,a,s'}\leftarrow p_{h,s,a,s'}, \forall (h,s,a,s)\in \mathcal{W}$;}
		\STATE{\quad $p'_{h,s,a,s'} \leftarrow 0, \forall (h,s,a,s')\notin \mathcal{W}$;}
		\STATE{\quad $p'_{h,s,a,z}\leftarrow \sum_{s':(h,s,a,s')\notin \mathcal{W}}p_{h,s,a,s'}, \forall (h,s,a)\in [H]\times \mathcal{S}\times \mathcal{A}$;}
		\STATE{\quad $p'_{h,z,a}\leftarrow \textbf{1}_{z}, \forall (h,a)\in [H]\times\mathcal{A}$;}
		\STATE{\quad \textbf{Return}: $p$.}
	\end{algorithmic}
\end{algorithm} 

\begin{algorithm}[tb]
	\caption{$\mathtt{Policy}$ $\mathtt{Elimination}$}
	\begin{algorithmic}[1]\label{alg:policy_elimination}	
		\STATE{\textbf{Input:} dataset $\mathcal{D}$, length $k$;  }
		\STATE{\textbf{Initialize:}  $\mathcal{D}^{0}\leftarrow \mathcal{D}$, $\mathcal{P}^{-1}\leftarrow (\Delta^{S})^{SA}$ $C_1\leftarrow 100 $, $v^{-1}_h(s)\leftarrow H-h+1$, $\forall (h,s)\in [H]\times \mathcal{S}$; $K_m \leftarrow \left\lceil K^{1-\frac{1}{2^m}}\right \rceil$  for $m=1,2,\ldots,M =\left\lceil \log_2\log_2 (K)\right\rceil$; }
			\STATE{$N_h(s,a,s')\leftarrow $ count of $(h,s,a,s')$ in $\mathcal{D}$;}
		\STATE{$\mathcal{W}\leftarrow \{(h,s,a,s'):\, N_h(s,a,s')\geq C_1H^2\iota\}$;}
		\FOR{$m = 0,1,2,\ldots, M-1$}
		\STATE{$\mathcal{P}^{m}\leftarrow \mathcal{P}^{m-1}\cap \mathtt{CR}^*\left(\mathcal{D}^{m},\overline{\mathcal{D}}^{m},\mathcal{W},\{v_{h}^{m-1}(s)\}_{(h,s)}\right)$;\label{line:cr}}
		\STATE{$\pi^{m+1}  \leftarrow \mathtt{Design}( (\mathcal{P}^{m} )$;\label{alg:eli_1}}
		\IF{$\sum_{m'=1}^{m}K_{m'}\leq k$}
		\STATE{Execute $\pi^{m+1}$ for $K_{m+1}$ episodes;}
		\ELSE
			\STATE{Execute $\pi^{m+1}$ for $k-(\sum_{m'=1}^{m}K_{m'})$ episodes;}
		\ENDIF
		\STATE{$\overline{D}^{m+1} \leftarrow$ the dataset in the $(m+1)$-th batch;}
		\STATE{Update the dataset $\mathcal{D}^{m+1}\leftarrow \mathcal{D}^{m}\cup \overline{\mathcal{D}}^{m+1}$;}
		\STATE{$v^m_h(s)\leftarrow \max_{\pi,p\in \mathcal{P}^m}\mathbb{E}_{\pi,p}\left[\sum_{h'=h}^{H}r_h(s_h,a_h)|s_h = s\right]$ for all $(h,s)\in [H]\times \mathcal{S}$;\label{line:v}}

		\ENDFOR

					\STATE{\vspace{1ex}
				\textbf{Function}: $\mathtt{CR}^*(\mathcal{D},\mathcal{D}',\mathcal{W},v)$:
				}
			\STATE{\quad $\{N_h(s,a,s')\}\leftarrow $counts in $\mathcal{D}$, $N_h(s,a)\leftarrow \max\{\sum_{s'}N_h(s,a,s'),1\}$ for all $(h,s,a,s')$;}
		\STATE{\quad $\hat{p}_{h,s,a,s'}\leftarrow \frac{N_h(s,a,s')}{N_h(s,a)}$, $\forall (h,s,a,s')$;}	
		\STATE{\quad $\{\check{N}_h(s,a,s')\}\leftarrow $ counts in $\mathcal{D}'$, $\check{N}_h(s,a)\leftarrow \max\{\sum_{s'}\check{N}_h(s,a,s'),1\}$ for all $(h,s,a,s')$;}
		\STATE{\quad $\check{p}_{h,s,a,s'}\leftarrow \frac{\check{N}_h(s,a,s')}{\check{N}_h(s,a)}$, $\forall (h,s,a,s')$;}
		\STATE{\quad $\tilde{\mathcal{P}}_{h,s,a}\leftarrow \Big\{p\in\Delta^{S}| \left|p_{s'}-\hat{p}_{h,s,a,s'}\right|\leq \alpha(N_{h}(s,a),N_h(s,a,s')) , \forall s'\in \mathcal{S}$,\\$ \quad \quad \quad \quad \quad \quad \quad \quad \quad \quad \quad \quad 
		 |(p-\check{p}_{h,s,a})v|\leq \alpha^*(\check{N}_{h}(s,a),\check{p}_{h,s,a},v) \Big\}$, $\forall (h,s,a)$;\label{line:alpha*}}
		\STATE{\quad $\mathcal{P}_{h,s,a}\leftarrow \{ \mathtt{clip}(p,\mathcal{W}): p\in \tilde{\mathcal{P}}_{h,s,a}\}$, $\forall (h,s,a)$;}
		\STATE{\quad \textbf{Return}: $\otimes_{h,s,a}\mathcal{P}_{h,s,a}$.}
			\STATE{\vspace{1ex}
				\textbf{Function}: $\mathtt{Design}(\mathcal{P})$:
				}
				\STATE{\quad $p\leftarrow$ arbitrary element in $\mathcal{P}$;}
\STATE{\quad \textbf{for} $i = 1,2,...,K^3$       \textbf{do}}
\STATE{\quad $\tilde{d}_h^{\tilde{\pi}^j}(s,a)\leftarrow W^{\tilde{\pi}^j}(\textbf{1}_{h,s,a},p)$ for $1\leq j \leq i-1$ and any $(h,s,a)$;}
\STATE{\quad $r^i_{h}(s,a)\leftarrow \min\left\{ \frac{1}{\sum_{j=1}^{i-1}\tilde{d}_{h}^{\tilde{\pi}^j}(s,a) },1\right\}$, $\forall (h,s,a)$;}
\STATE{\quad $\tilde{\pi}^{i}\leftarrow \mathtt{Policy\,Search}(r,r^i,\mathcal{P})$;}
\STATE{\quad \textbf{end for}}
\STATE{\quad $\{\pi,p\}\leftarrow \mathtt{Sum}\left(\left\{\frac{1}{K^3},\tilde{\pi}^i,p\right\}_{i=1}^{K^3} \right)$;}
		\STATE{\quad \textbf{Return}: $\pi$.}
			\end{algorithmic}
\end{algorithm}

\subsection{Raw Exploration}
Given a dataset $\mathcal{D}$ with counts $\{N_h(s,a,s')\}$, we define the set of \emph{known} tuples as $\{(h,s,a,s'):N_h(s,a,s')\geq C_1H^2\iota\}$ and the left tuples are regarded as \emph{infrequent} tuples.  

In Algorithm~\ref{alg:raw_exploration}, we are given a dataset. Then we compute the corresponding confidence region $\mathcal{P}$ in Line~\ref{line:alpha}, where $\alpha(n,n') = \sqrt{\frac{4n'\iota}{n^2}}+\frac{5\iota}{n} .  $

We conduct exploration layer by layer over policies in the set of survived policies $\Pi(r,\mathcal{P})$. By visiting each $(h,s,a)$ as much as possible, we can judge whether a tuple $(h,s,a,s')$ is hard to visit using policies in $\Pi(r,\mathcal{P})$. 

Given the set of $\emph{known}$ tuples $\mathcal{W}$, we redirect all tuples not in $\mathcal{W}$ to an additional absorbed state $z$ using $\mathtt{clip}(\cdot,\cdot)$. Once we prove that the probability of reaching $z$ is small enough for the any optimal policy, we can directly learn under the clipped transition model.

In Line~\ref{line:pss} Algorithm~\ref{alg:raw_exploration}, the algorithm $\mathtt{Policy\, Search}$ is invoked. Given any reward $u,u'$, any confidence region $\mathcal{P}$ and threshold $\epsilon>0$, 
this algorithm returns a policy $\tilde{\pi}\in \Pi(u,\mathcal{P})$ such that $W^{\tilde{\pi}}(u',p)\geq c\max_{\pi\in \Pi(u,\mathcal{P})}W^{\pi}(u',p)-\epsilon$ with 
some universal constant $c>0$.
Moreover, when $\mathcal{P}$ is \emph{tight} w.r.t. the true transition model $P$ after clipping, the time complexity of the algorithm is $O(\mathrm{poly}(SAHK)\log(1/\epsilon))$.  
The algorithm and corresponding analysis is postponed to Appendix~\ref{sec:cei}. 

It is also worth noting that executing each $\pi_{h,s,a}$ with probability $\frac{1}{SA}$ can not be regarded as a (history-independent) policy because  the agent need to keep in mind which policy is chosen in current episode. In contrast, the agent only needs to observe current state to take actions following a policy. To address this problem, we define an operator $\mathtt{Sum}$ to take sum over policies under some transition model. Formally, we have the lemma below and postpone the proof to Appendix~\ref{app:pflemma2}.

\begin{lemma}\label{lemma:stp}
		Let $\mathcal{P} =\otimes_{(h,s,a)} \mathcal{P}_{h,s,a}$ be a set of transition models such that $\mathcal{P}_{h,s,a}\subset \Delta^{S}$ is convex for any $(h,s,a)$. Let
	$\{(\pi^i,P^i)\}_{i=1}^n$ be a sequence of policy-transition pairs such that $P^i\in \mathcal{P}$.
	For any $\{\lambda_{i}\}_{i=1}^n$ such that $\lambda_i\geq 0$ for $i\geq 1$ and $\sum_i \lambda_i = 1$, there exists a policy $\pi$ and $P\in \mathcal{P}$, satisfying that
	\begin{align}
	W^{\pi}(\textbf{1}_{h,s,a}, P)= \sum_i \lambda_i W^{\pi^i}(\textbf{1}_{h,s,a}, P^i)
	\end{align}
	for any $(h,s,a)\in [H]\times \mathcal{S}\times\mathcal{A}$. Furthermore, the time complexity to find $\{\pi,P\}$ could be bounded by $O(nS^3A^2H^2)$.
\end{lemma}
Therefore, for any $\{ \lambda_i, \pi^i, P^i \}_{i=1}^n$ satisfying $\sum_{i=1}^n \lambda_i = 1$ and $\lambda_i \geq 0$ for $i\geq 1$ as input, there exists $\{\pi, P\}$ such that
$	W^{\pi}(\textbf{1}_{h,s,a}, P)= \sum_i \lambda_i W^{\pi^i}(\textbf{1}_{h,s,a}, P^i)$
	and $P_{h,s,a}\in \mathrm{Convex}(\{P_{h,s,a}^i \}_{i=1}^n)$
	for any $(h,s,a)\in [H]\times \mathcal{S}\times\mathcal{A}$, where $\mathrm{Convex}(\mathcal{U})$ denotes the convex hull of the set  $\mathcal{U}$. Then $\mathtt{Sum}$ is defined as $\mathtt{Sum}(\{ \lambda_i, \pi^i, P^i \}_{i=1}^n )= \{\pi,P\}$.

\subsection{Policy Elimination}

Given the dataset collected in the first two stages, we first compute the $\emph{known}$ set $\mathcal{W}$. Unlike  Algorithm~\ref{alg:raw_exploration}, we do not update $\mathcal{W}$ in the rest time because the first two stages can ensure that the probability of visiting $\mathcal{W}^{C}$ is $O(1/\sqrt{K})$. 

As mentioned in Section~\ref{sec:tec}, for each batch, we invoke reward-zero exploration to search for the policy with near-optimal coverage. Based on such a policy, we can provide uniform bound for the length of confidence intervals for all survived policies, which enables us to using the batch sizes in bandit algorithms \citep{perchet2016batched,gao2019batched}.  

Besides, to obtain a better regret bound, we estimate the optimal value function at the end of each batch, and use it to build a tighter confidence region. 
As presented in Line~\ref{line:alpha*} Algorithm~\ref{alg:policy_elimination}, we use two empirical transition probabilities to construct the confidence region.  Noting that the samples in the $m$-th batch is independent of $v^{m-1}$, 
we could add a Bernstein-style constraint, where $
    \alpha^*(n,p,v) =   5\sqrt{\frac{\mathbb{V}(p,v)\iota}{n}} +       \frac{3\iota}{n}.\label{eq:defalpha*}$.

\section{Conclusion}\label{sec:con}

In this paper, we study multi-batch reinforcement learning, and provide an efficient algorithm to achieve the near-optimal regret bound and batch complexity. It would be an interesting problem to generalize our results to reinforcement learning with function approximation case, e.g., linear MDP. Another important direction is to study the exact batch-regret trade-off  for multi-batch reinforcement learning.

	\paragraph{Broader Impact} This work focus on the theory of multi-batch reinforcement learning, and the broader impact is not applicable.

\bibliography{reference}

\begin{thebibliography}{28}
\providecommand{\natexlab}[1]{#1}
\providecommand{\url}[1]{\texttt{#1}}
\expandafter\ifx\csname urlstyle\endcsname\relax
  \providecommand{\doi}[1]{doi: #1}\else
  \providecommand{\doi}{doi: \begingroup \urlstyle{rm}\Url}\fi

\bibitem[Almirall et~al.(2012)Almirall, Compton, Gunlicks-Stoessel, Duan, and
  Murphy]{almirall2012designing}
Daniel Almirall, Scott~N Compton, Meredith Gunlicks-Stoessel, Naihua Duan, and
  Susan~A Murphy.
\newblock Designing a pilot sequential multiple assignment randomized trial for
  developing an adaptive treatment strategy.
\newblock \emph{Statistics in medicine}, 31\penalty0 (17):\penalty0 1887--1902,
  2012.

\bibitem[Almirall et~al.(2014)Almirall, Nahum-Shani, Sherwood, and
  Murphy]{almirall2014introduction}
Daniel Almirall, Inbal Nahum-Shani, Nancy~E Sherwood, and Susan~A Murphy.
\newblock Introduction to smart designs for the development of adaptive
  interventions: with application to weight loss research.
\newblock \emph{Translational behavioral medicine}, 4\penalty0 (3):\penalty0
  260--274, 2014.

\bibitem[Azar et~al.(2013)Azar, Munos, and Kappen]{azar2013minimax}
Mohammad~Gheshlaghi Azar, R{\'e}mi Munos, and Hilbert~J Kappen.
\newblock Minimax {PAC} bounds on the sample complexity of reinforcement
  learning with a generative model.
\newblock \emph{Machine learning}, 91\penalty0 (3):\penalty0 325--349, 2013.

\bibitem[Azar et~al.(2017)Azar, Osband, and Munos]{azar2017minimax}
Mohammad~Gheshlaghi Azar, Ian Osband, and R{\'e}mi Munos.
\newblock Minimax regret bounds for reinforcement learning.
\newblock In \emph{Proceedings of the 34th International Conference on Machine
  Learning-Volume 70}, pages 263--272. JMLR. org, 2017.

\bibitem[Bai et~al.(2019)Bai, Xie, Jiang, and Wang]{bai2019provably}
Yu~Bai, Tengyang Xie, Nan Jiang, and Yu-Xiang Wang.
\newblock Provably efficient q-learning with low switching cost.
\newblock In \emph{Advances in Neural Information Processing Systems}, pages
  8004--8013, 2019.

\bibitem[Bartlett and Tewari(2009)]{bartlett2009regal}
Peter~L Bartlett and Ambuj Tewari.
\newblock Regal: a regularization based algorithm for reinforcement learning in
  weakly communicating mdps.
\newblock In \emph{Proceedings of the 25th Conference on Uncertainty in
  Artificial Intelligence (UAI 2009))}, 2009.

\bibitem[Cesa-Bianchi et~al.(2013)Cesa-Bianchi, Dekel, and
  Shamir]{cesa2013online}
Nicolo Cesa-Bianchi, Ofer Dekel, and Ohad Shamir.
\newblock Online learning with switching costs and other adaptive adversaries.
\newblock In \emph{Advances in Neural Information Processing Systems}, pages
  1160--1168, 2013.

\bibitem[Cohen et~al.(2021)Cohen, Lee, and Song]{cohen2021solving}
Michael~B Cohen, Yin~Tat Lee, and Zhao Song.
\newblock Solving linear programs in the current matrix multiplication time.
\newblock \emph{Journal of the ACM (JACM)}, 68\penalty0 (1):\penalty0 1--39,
  2021.

\bibitem[Dann et~al.(2019)Dann, Li, Wei, and Brunskill]{dann2019policy}
Christoph Dann, Lihong Li, Wei Wei, and Emma Brunskill.
\newblock Policy certificates: Towards accountable reinforcement learning.
\newblock In \emph{Proceedings of the 36th International Conference on Machine
  Learning}, volume~97 of \emph{Proceedings of Machine Learning Research},
  pages 1507--1516, Long Beach, California, USA, 09--15 Jun 2019. PMLR.

\bibitem[Dong et~al.(2020)Dong, Li, Zhang, and Zhou]{dong2020multinomial}
Kefan Dong, Yingkai Li, Qin Zhang, and Yuan Zhou.
\newblock Multinomial logit bandit with low switching cost.
\newblock In \emph{International Conference on Machine Learning}, pages
  2607--2615. PMLR, 2020.

\bibitem[Gao et~al.(2021)Gao, Xie, Du, and Yang]{gao2021provably}
Minbo Gao, Tianle Xie, Simon~S Du, and Lin~F Yang.
\newblock A provably efficient algorithm for linear markov decision process
  with low switching cost.
\newblock \emph{arXiv preprint arXiv:2101.00494}, 2021.

\bibitem[Gao et~al.(2019)Gao, Han, Ren, and Zhou]{gao2019batched}
Zijun Gao, Yanjun Han, Zhimei Ren, and Zhengqing Zhou.
\newblock Batched multi-armed bandits problem.
\newblock \emph{arXiv preprint arXiv:1904.01763}, 2019.

\bibitem[Jaksch et~al.(2010)Jaksch, Ortner, and Auer]{jaksch2010near}
Thomas Jaksch, Ronald Ortner, and Peter Auer.
\newblock Near-optimal regret bounds for reinforcement learning.
\newblock \emph{Journal of Machine Learning Research}, 11\penalty0
  (Apr):\penalty0 1563--1600, 2010.

\bibitem[Jin et~al.(2018)Jin, Allen-Zhu, Bubeck, and Jordan]{jin2018q}
Chi Jin, Zeyuan Allen-Zhu, Sebastien Bubeck, and Michael~I Jordan.
\newblock Is {Q}-learning provably efficient?
\newblock In \emph{Advances in Neural Information Processing Systems}, pages
  4863--4873, 2018.

\bibitem[Kakade(2003)]{kakade2003sample}
Sham~M Kakade.
\newblock \emph{On the sample complexity of reinforcement learning}.
\newblock PhD thesis, University of London London, England, 2003.

\bibitem[Krishnan et~al.(2018)Krishnan, Yang, Goldberg, Hellerstein, and
  Stoica]{krishnan2018learning}
Sanjay Krishnan, Zongheng Yang, Ken Goldberg, Joseph Hellerstein, and Ion
  Stoica.
\newblock Learning to optimize join queries with deep reinforcement learning.
\newblock \emph{arXiv preprint arXiv:1808.03196}, 2018.

\bibitem[Lei et~al.(2012)Lei, Nahum-Shani, Lynch, Oslin, and
  Murphy]{lei2012smart}
Huitan Lei, Inbal Nahum-Shani, Kevin Lynch, David Oslin, and Susan~A Murphy.
\newblock A" smart" design for building individualized treatment sequences.
\newblock \emph{Annual review of clinical psychology}, 8:\penalty0 21--48,
  2012.

\bibitem[Li et~al.(2020)Li, Wei, Chi, Gu, and Chen]{li2020breaking}
Gen Li, Yuting Wei, Yuejie Chi, Yuantao Gu, and Yuxin Chen.
\newblock Breaking the sample size barrier in model-based reinforcement
  learning with a generative model.
\newblock \emph{arXiv preprint arXiv:2005.12900}, 2020.

\bibitem[Mirhoseini et~al.(2017)Mirhoseini, Pham, Le, Steiner, Larsen, Zhou,
  Kumar, Norouzi, Bengio, and Dean]{mirhoseini2017device}
Azalia Mirhoseini, Hieu Pham, Quoc~V Le, Benoit Steiner, Rasmus Larsen, Yuefeng
  Zhou, Naveen Kumar, Mohammad Norouzi, Samy Bengio, and Jeff Dean.
\newblock Device placement optimization with reinforcement learning.
\newblock In \emph{International Conference on Machine Learning}, pages
  2430--2439. PMLR, 2017.

\bibitem[Perchet et~al.(2016)Perchet, Rigollet, Chassang, Snowberg,
  et~al.]{perchet2016batched}
Vianney Perchet, Philippe Rigollet, Sylvain Chassang, Erik Snowberg, et~al.
\newblock Batched bandit problems.
\newblock \emph{Annals of Statistics}, 44\penalty0 (2):\penalty0 660--681,
  2016.

\bibitem[Qiao et~al.(2022)Qiao, Yin, Min, and Wang]{qiao2022sample}
Dan Qiao, Ming Yin, Ming Min, and Yu-Xiang Wang.
\newblock Sample-efficient reinforcement learning with loglog (t) switching
  cost.
\newblock \emph{arXiv preprint arXiv:2202.06385}, 2022.

\bibitem[Ruan et~al.(2020)Ruan, Yang, and Zhou]{ruan2020linear}
Yufei Ruan, Jiaqi Yang, and Yuan Zhou.
\newblock Linear bandits with limited adaptivity and learning distributional
  optimal design.
\newblock \emph{arXiv preprint arXiv:2007.01980}, 2020.

\bibitem[Simchi-Levi and Xu(2019)]{simchi2019phase}
David Simchi-Levi and Yunzong Xu.
\newblock Phase transitions and cyclic phenomena in bandits with switching
  constraints.
\newblock \emph{Available at SSRN 3380783}, 2019.

\bibitem[Yu et~al.(2019)Yu, Yang, Kolar, and Wang]{yu2019convergent}
Ming Yu, Zhuoran Yang, Mladen Kolar, and Zhaoran Wang.
\newblock Convergent policy optimization for safe reinforcement learning.
\newblock \emph{arXiv preprint arXiv:1910.12156}, 2019.

\bibitem[Zanette and Brunskill(2019)]{zanette2019tighter}
Andrea Zanette and Emma Brunskill.
\newblock Tighter problem-dependent regret bounds in reinforcement learning
  without domain knowledge using value function bounds.
\newblock In \emph{International Conference on Machine Learning}, pages
  7304--7312, 2019.

\bibitem[Zhang and Ji(2019)]{zhang2019regret}
Zihan Zhang and Xiangyang Ji.
\newblock Regret minimization for reinforcement learning by evaluating the
  optimal bias function.
\newblock In \emph{Advances in Neural Information Processing Systems}, pages
  2823--2832, 2019.

\bibitem[Zhang et~al.(2020)Zhang, Zhou, and Ji]{zhang2020almost}
Zihan Zhang, Yuan Zhou, and Xiangyang Ji.
\newblock Almost optimal model-free reinforcement learning via
  reference-advantage decomposition.
\newblock \emph{arXiv preprint arXiv:2004.10019}, 2020.

\bibitem[Zhang et~al.(2021)Zhang, Ji, and Du]{zhang2021reinforcement}
Zihan Zhang, Xiangyang Ji, and Simon Du.
\newblock Is reinforcement learning more difficult than bandits? a near-optimal
  algorithm escaping the curse of horizon.
\newblock In \emph{Conference on Learning Theory}, pages 4528--4531. PMLR,
  2021.

\end{thebibliography}
	\bibliographystyle{plainnat}

\medskip

\section*{Checklist}

The checklist follows the references.  Please
read the checklist guidelines carefully for information on how to answer these
questions.  For each question, change the default \answerTODO{} to \answerYes{},
\answerNo{}, or \answerNA{}.  You are strongly encouraged to include a {\bf
justification to your answer}, either by referencing the appropriate section of
your paper or providing a brief inline description.  For example:
\begin{itemize}
  \item Did you include the license to the code and datasets? \answerNA{The paper is theoretical and there is no numerical experiments.}
  \item Did you include the license to the code and datasets? \answerNA{The paper is theoretical and there is no numerical experiments.}
  \item Did you include the license to the code and datasets?\answerNA{The paper is theoretical and there is no numerical experiments.}
\end{itemize}
Please do not modify the questions and only use the provided macros for your
answers.  Note that the Checklist section does not count towards the page
limit.  In your paper, please delete this instructions block and only keep the
Checklist section heading above along with the questions/answers below.

\begin{enumerate}

\item For all authors...
\begin{enumerate}
  \item Do the main claims made in the abstract and introduction accurately reflect the paper's contributions and scope?
    \answerYes{We provide a near-optimal regret bound for multi-batch RL}
  \item Did you describe the limitations of your work?
    \answerYes{We focus on studying the tabular case. More efforts are required to extend the results to RL with function approximation}
  \item Did you discuss any potential negative societal impacts of your work?
    \answerNA{The paper is theoretical and there is no possible negative societal impacts.}
  \item Have you read the ethics review guidelines and ensured that your paper conforms to them?
    \answerYes{}.
\end{enumerate}

\item If you are including theoretical results...
\begin{enumerate}
  \item Did you state the full set of assumptions of all theoretical results?
    \answerYes{}
        \item Did you include complete proofs of all theoretical results?
    \answerYes{We sketch the proof in the main body. The details are postpone to the appendix}
\end{enumerate}

\item If you ran experiments...
\begin{enumerate}
  \item Did you include the code, data, and instructions needed to reproduce the main experimental results (either in the supplemental material or as a URL)?
\answerNA{The paper is theoretical and there is no numerical experiments.}
  \item Did you specify all the training details (e.g., data splits, hyperparameters, how they were chosen)?
  \answerNA{}
        \item Did you report error bars (e.g., with respect to the random seed after running experiments multiple times)?
    \answerNA{}
        \item Did you include the total amount of compute and the type of resources used (e.g., type of GPUs, internal cluster, or cloud provider)?
    \answerNA{}
\end{enumerate}

\item If you are using existing assets (e.g., code, data, models) or curating/releasing new assets...
\begin{enumerate}
  \item If your work uses existing assets, did you cite the creators?
\answerNA{The paper is theoretical and there is no numerical experiments.}
  \item Did you mention the license of the assets?
    \answerNA{}
  \item Did you include any new assets either in the supplemental material or as a URL?
    \answerNA{}
  \item Did you discuss whether and how consent was obtained from people whose data you're using/curating?
    \answerNA{}
  \item Did you discuss whether the data you are using/curating contains personally identifiable information or offensive content?
    \answerNA{}
\end{enumerate}

\item If you used crowdsourcing or conducted research with human subjects...
\begin{enumerate}
  \item Did you include the full text of instructions given to participants and screenshots, if applicable?
    \answerNA{This paper is irrelevant to crowdsourcing or human projects.}
  \item Did you describe any potential participant risks, with links to Institutional Review Board (IRB) approvals, if applicable?
    \answerNA{}
  \item Did you include the estimated hourly wage paid to participants and the total amount spent on participant compensation?
    \answerNA{}
\end{enumerate}

\end{enumerate}


\newpage
\appendix

	\section{Technical Lemmas}
	\begin{lemma}\label{lemma:bennet}
		Let $Z,Z_1,...,Z_n$  be i.i.d. random variables with values in $[0,1]$ and let $\delta>0$. Define $\mathbb{V}Z = \mathbb{E}\left[(Z-\mathbb{E}Z)^2 \right]$. Then we have
		\begin{align}
		\mathbb{P}\left[ \left|\mathbb{E}\left[Z\right]-\frac{1}{n}\sum_{i=1}^n Z_i  \right| > \sqrt{\frac{  2\mathbb{V}Z \ln(2/\delta)}{n}} +\frac{\ln(2/\delta)}{n} \right]\leq \delta.\nonumber
		\end{align}
	\end{lemma}

\begin{lemma}\label{lemma:ete}
Let $X_1,X_2,\ldots$ be a sequence of random variables taking value in $[0,l]$. Define $\mathcal{F}_k =\sigma(X_1,X_2,\ldots,X_{k-1})$ and $Y_k = \mathbb{E}[X_k|\mathcal{F}_k]$ for $k\geq 1$. For any $\delta>0$, we have that
\begin{align}
& \mathbb{P}\left[ \exists n, \sum_{k=1}^n X_k \leq  3\sum_{k=1}^n Y_k+ l\ln(1/\delta)\right]\leq \delta\nonumber
\\  & \mathbb{P}\left[  \exists n,  \sum_{k=1}^n Y_k \geq 3\sum_{k=1}^n X_k + l\ln(1/\delta)  \right]    \leq \delta .\nonumber 
\end{align}
\end{lemma}

\begin{proof} Let $t\in [0,1/l]$ be fixed.
Consider to bound $Z_k:=\mathbb{E}[\exp(t\sum_{k'=1}^k(X_{k'}-3Y_{k'})  )]$. By definition, we have that
\begin{align}
    \mathbb{E}[Z_k|\mathcal{F}_k]  & =\exp(t\sum_{k'=1}^{k}(X_{k'}-3Y_{k'})) \mathbb{E}\left[t(X_{k}-3Y_{k})\right] \nonumber 
    \\ & \leq \exp(t\sum_{k'=1}^{k}(X_{k'}-3Y_{k'}))\exp(3Y_{k})\cdot \mathbb{E}[1+tX_k+2t^2X^2_{k}]\nonumber 
    \\ & \leq \exp(t\sum_{k'=1}^{k}(X_{k'}-3Y_{k'}))\exp(3Y_{k})\cdot \mathbb{E}[1+3tX_k]\nonumber 
    \\ & = \exp(t\sum_{k'=1}^{k}(X_{k'}-3Y_{k'}))\exp(3Y_{k})\cdot (1+3tY_k)\nonumber 
    \\ & \leq \exp(t\sum_{k'=1}^{k}(X_{k'}-3Y_{k'}))\nonumber
    \\ & =Z_{k-1},\nonumber 
\end{align}
where the second line is by the fact that $e^x\leq 1+x+2x^2$ for $x\in [0,1]$. 
Define $Z_0 = 1$
Then $\{Z_{k}\}_{k\geq 0}$ is a super-martingale with respect to $\{\mathcal{F}_{k}\}_{k\geq 1}$. Let $\tau$ be the smallest $n$ such that $\sum_{k=1}^n X_k - 3\sum_{k=1}^n Y_k >l\ln(1/\delta)$.
It is easy to verify that $Z_{\min\{\tau,n\}}\leq \exp(tl\ln(1/\delta)+tl)<\infty$. Choose $t=1/l$.
By the optimal stopping time theorem, we have that
\begin{align}
 & \mathbb{P}\left[ \exists n\leq N, \sum_{k=1}^n X_k \geq  3\sum_{k=1}^n Y_k + l\ln(1/\delta)\right]\nonumber 
 \\ & =\mathbb{P}\left[\tau \leq N\right]\nonumber 
 \\ & \leq \mathbb{P}\left[ Z_{\min\{\tau,N\}}\geq \exp(tl\ln(1/\delta))  \right]\nonumber 
 \\ & \leq \frac{\mathbb{E}[Z_{\min\{\tau,N\}}]}{\exp(tl\ln(1/\delta))}\nonumber 
 \\ & \leq \delta. \nonumber
\end{align}

Letting $N\to \infty$, we have that
\begin{align}
& \mathbb{P}\left[ \exists n, \sum_{k=1}^n X_k \leq  3\sum_{k=1}^n Y_k+ l\ln(1/\delta)\right]\leq \delta.\nonumber
\end{align}
Considering $W_k = \mathbb{E}[\exp(t\sum_{k'=1}^k (Y_k/3-X_k))]$, using similar arguments  and choosing $t=1/(3l)$, we have that
\begin{align}
 & \mathbb{P}\left[  \exists n,  \sum_{k=1}^n Y_k \geq 3\sum_{k=1}^n X_k + l\ln(1/\delta)  \right]    \leq \delta .\nonumber 
\end{align}
The proof is completed.
\end{proof}

	\begin{lemma}\label{lemma:pd} Let the policy $\pi$ and reward $r$ be fixed. Let  $p$ and $p'$ be two transition model, it holds that
	\begin{align}
	    W^{\pi}(r,p)-W^{\pi}(r,p') = \sum_{h,s,a} W^{\pi}(\textbf{1}_{h,s,a},p) (p'_{h,s,a}-p_{h,s,a})V'_{h+1},
	\end{align}
	where $\{V'_{h}(s)\}_{(h,s)\in [H]\times \mathcal{S}}$ is the value function under $p'$ following $\pi$.
	\end{lemma}

\section{Lower Bound (Proof of Theorem~\ref{thm:lb})}\label{app:lb}
 Firstly, by the lower bound on batched bandit (Theorem 3 in \citep{gao2019batched}), to achieve $O(\mathrm{poly}(S,A,H)\sqrt{K})$ regret, the number of batches is at least $\Omega(\log_2\log_2(K))$. To show a lower bound of $\Omega(H/\log_A(K))$, we have the lemma below by considering an MDP with $2$ states and $A$ actions. 
\begin{lemma}\label{lemma:lb1}
Let $\mathcal{S}=\{s^{(0)}, s^{(1)}\}$, $\mathcal{A}=\{a_0,a_1,\ldots,a_{A}\}$ and $s_{1}=s^{(0)}$. Let $d=\left\lfloor 2\log_A(K)\right \rfloor+2$.  For $v=[v_1,v_2,\ldots,v_d]^{\top}\in A^d$, we define the transition model $P^{v}$ by setting $P^{v}_{h,s^{(0)},a_{x}}=[1,0]^{\top},\forall x\neq v_h$, $P^{v}_{h,s^{(0)},a_{v_h}}=[0,1]^{\top}$ and $P^{v}_{h,s^{(1)},a_{x}}=[0,1]^{\top},\forall 1\leq x\leq A$  for $1\leq h\leq d$. Let $\pi$ be a stochastic policy, Then there exists $v$ such that with probability $1-\frac{1}{K}$, $(h,s^{(0)})$ is never visited in $K$ episodes following $\pi$.
\end{lemma}
\begin{proof}
Denote the distribution of $\pi$ as $\mathtt{D}$, we define $x=[x_1,x_2,\ldots,x_d]^{\top}$ as below. 
Let $x_1 =\arg\min_{i}\mathbb{E}_{\pi\sim D}\left[\pi_1(a_i|s^{(0)})\right] $. For $2\leq h\leq d$, we define 
\begin{align} 
x_h = \arg\max_i\frac{\mathbb{E}_{\pi\sim \mathtt{D}}\left[\mathbb{I}_{\pi}[s_{h-1} = s^{(0)} | P^{x,h-1}]\pi_h(a_{i}|s_h) \right]  }{\mathbb{E}_{\pi\sim \mathtt{D}}\left[\mathbb{I}_{\pi}[s_{h-1} = s^{(0)} | P^{x,h-1}] \right] }  ,\nonumber
\end{align}
where $P^{x,h-1}$ denote the first $(h-1)$-layers of the transition model $P^{x}$. Because $\mathbb{P}_{\pi}[s_h = s^{(0)} | P^{x}]$ is determined by the first $(h-1)$-layers of $P^{x}$, $x_h$ is well-defined.
By definition we have that 
\begin{align}
  \frac{\mathbb{E}_{\pi\sim \mathtt{D}}\left[\mathbb{I}_{\pi}[s_{h-1} = s^{(0)} | P^{x,h-1}]\pi_h(a_{x_h}|s_h) \right]  }{\mathbb{E}_{\pi\sim \mathtt{D}}\left[\mathbb{I}_{\pi}[s_{h-1} = s^{(0)} | P^{x,h-1}] \right] } \leq \frac{1}{A}.\label{eq:scws}
\end{align}
Recall that $x = [x_1,x_2,\ldots, x_d]^{\top}$. For $1\leq h'\leq d$, by \eqref{eq:scws} we have that
\begin{align}
  &\mathbb{E}_{\pi\sim \mathtt{D}}\mathbb{P}_{\pi}\left[ s_{h'} = s^{(0)} | P^{x}\right]\nonumber \\ & = \mathbb{E}_{\pi\sim \mathtt{D}} \Pi_{h=1}^{h'-1}\pi_{h}(a_{x_h}| s^{(0)}) \nonumber
  \\ & = \mathbb{E}_{\pi\sim \mathtt{D}}\mathbb{P}_{\pi}\left[ s_{h'-1} = s^{(0)} | P^{x}\right] \cdot \frac{\mathbb{E}_{\pi\sim \mathtt{D}}\left[\mathbb{I}_{\pi}[s_{h'-1} = s^{(0)} | P^{x,h-1}]\pi_h(a_{x_h}|s_h) \right]  }{\mathbb{E}_{\pi\sim \mathtt{D}}\left[\mathbb{I}_{\pi}[s_{h'-1} = s^{(0)} | P^{x,h-1}] \right] } \nonumber
  \\ & \leq \frac{1}{A} \mathbb{E}_{\pi\sim \mathtt{D}}\mathbb{P}_{\pi}\left[ s_{h'-1} = s^{(0)} | P^{x}\right].
\end{align}
Therefore, $\mathbb{E}_{\pi\sim \mathtt{D}}\mathbb{P}_{\pi}\left[ s_{d} = s^{(0)} | P^{x}\right]\leq \frac{1}{A^{d-1}}\leq \frac{1}{K}$.
Then the probability of visiting $(h,s^{(0)})$ in $K$ episodes is at most $\frac{1}{K}$, where the conclusion follows.
\end{proof}
We name the MDP in Lemma~\ref{lemma:lb1} as a basic MDP.
Now we construct our counter-example by concatenating $\Theta(H/\log_A(K))$ basic MDPs and a tail MDP with large rewards.
Let $\mathcal{S}=\{s^{(0)}, s^{(1)}\}$ and $\mathcal{A}=\{a_1,a_2,\ldots,a_{A}\}$. Let $d=\left\lfloor 2\log_A(K)\right\rfloor+2$ and  $c= \left\lfloor\frac{H}{2d}\right\rfloor$. Then $c =C'H/\log_A(K)$ for some constant $C'$. For $v=[v_{1},v_{2},\ldots, v_{cd}]^{\top}\in \{0,1\}^{cd}$, we define the transition model $P^{v}$ as below: $P^{v}_{id+j,s^{(0)},a_{v_{id+j}}}=[1,0]^{\top}$, $P^{v}_{id+j,s^{(0)},a_{l}}=[0,1]^{T}$ for $l\neq v_{id+j}$ and $P^{v}_{id+j,s^{(1)},a_{l}}=[0,1]^{\top}$ for $1\leq l \leq A$ for any $0\leq i \leq c-1$ and $1\leq j\leq d$;  $P_{h,s^{(0)},a_{l}}=[1,0]^{\top}$ and $P_{h,s^{(1)},a_{l}}=[0,1]^{\top}$ for any $1\leq l \leq A$ and $cd+1\leq h\leq H$. The reward function $r$ is given by $r_{h,s^{(0)},a_{l}}$ for $1\leq l\leq A$ $cd+1\leq h\leq H$ and $0$ for other $(h,s,a)$ triples.

To achieve sub-linear regret, the agent needs to visit $(cd+1,s^{(0)})$ for at least one time. Then the proof is completed by the lemma below.
\begin{lemma}\label{lemma:lb2}
If the number of batches $M\leq c-2$, for any algorithm $\mathcal{G}$ there exists $v$ such that with probability $1-\frac{c}{K}\geq \frac{1}{2}$, $(cd+1,s^{(0)})$ is never visited.
\end{lemma}
\begin{proof}
Let $m_i$ denote the number of batches used at the time when $(id+1,s^{(0)})$ is visited for the first time. Besides, we let $\pi(i)$ denote the policy at time $m_i$. Because $\pi(i)$ is determined before visiting $(id+1,s^{(0)})$, given the algorithm $\mathcal{G}$, $\pi^{i}$ could be viewed as a stochastic function of $\{v_1,v_2,\ldots, v_{id} \}$. By Lemma~\ref{lemma:lb1}, when $\{v_1,v_2,\ldots, v_{id} \}$ is fixed, we can choose $\{v_{id+1},\ldots, v_{id+d} \}$ properly, so that with probability $1-\frac{1}{K}$, $((i+1)d+1, s^{(0)})$ is never visited in $K$ episodes following $\pi(i)$. Therefore, with probability $1-\frac{1}{K}$, $\pi(i+1)\neq \pi(i)$, which implies that $m_{i+1}\geq m_{i}+1$. By choosing $\{v_{id+1},v_{id+1},\ldots,v_{id+d} \}$ recursively following the way in Lemma~\ref{lemma:lb1} for $0\leq i \leq c-1$,  we have that with probability $1-\frac{c}{K}$, $m_{i+1}\geq m_{i}+1$ for $1\leq i\leq c$, where $m_{c}\geq c-1$ follows. Then the conclusion follows by the equation below.
\begin{align}
    \mathbb{P}\left[M\leq c-2, \,(cd+1,s^{(0)}) \text{ is visited} \right]= \mathbb{P}\left[ m_{c}\leq c-2 \right]\leq \frac{c}{M}. \nonumber
\end{align}
\end{proof}

\section{Efficient Implementation of the Proposed Algorithm}\label{sec:cei}

In this section, we analyze the computational cost of Algorithm~\ref{alg:main}. In particular, we first introduce the algorithm $\mathtt{PolicySearch}$ to show that it can help find the desired exploration policy efficiently. 

\subsection{The Algorithm}
$\mathtt{Policy\,Search}$ is presented in Algorithm~\ref{alg:ps}. The algorithms takes two reward functions $u,u'$ and a confidence region $\mathcal{P}$ as input, and output a policy $\pi$ and $p\in \mathcal{P}$ such that
$W^{\pi}(u',P) $ is large enough compared to $ \max_{\pi'\in \Pi(u,{\mathcal{P}}),}W^{\pi'}(u',P) $. 

In the algorithm, we first compute $a:=\max_{\pi}U^{\pi}(u+\textbf{1}_{z},\mathcal{P})$ and $b: =\max_{\pi}L^{\pi}(u,\mathcal{P}) $. Then we set the target reward as $u+\textbf{1}_{z}+\eta u'$ for different $\eta$ and learn the corresponding optimal policy and transition model $\{\pi^{\eta},P^{\eta}\}$. In intuition, the larger $\eta$ is, the larger $W^{\pi^{\eta}}(u,P^{\eta})$ is. In this way, we aim  to find the maximal $\eta$ such that $\pi^{\eta}$ is not eliminated, i.e., $\pi^{\eta}\in \Pi(u,\mathcal{P})$. To find such $\eta$, we play the naive dichotomy method as presented in Algorithm~\ref{alg:ps}.

When $u = r$, we assume that $a-b\geq \frac{1}{K^3}$ without loss of generality. 
Note that when $a-b\leq \frac{1}{K^3}$, any policy $\pi$ in $\Pi(r,\mathcal{P})$ is $\frac{1}{K^3}$ optimal and we can follow $\pi$ in the rest episodes.

 In Algorithm~\ref{alg:ps}, we invoke extended value iteration (EVI, see Algorithm~\ref{alg:evi}) as a sub-routine. Algorithm~\ref{alg:evi} targets compute $(\pi,p)\leftarrow \arg\max_{\pi,p\in \mathcal{P}}W^{\pi}(u,p)$ for some reward function $u$ and confidence region $\mathcal{P}$. In finite-horizon MDP, this step could be implemented by back induction. So it suffices to solve $\arg\max_{a,p\in \mathcal{P}} p_{h,s,a}V_{h+1} $ where $V_{h+1}$ is the value function computed by back induction. Note that in this paper, the confidence region could be described by at most $O(S^2AK)$ linear constraints, which enables us to find an approximate solution in polynomial time. 
Besides, given $u$ and $\mathcal{P}$,  $\max{\pi}U^{\pi}(u,\mathcal{P})$ and $\max_{\pi}L^{\pi}(u,\mathcal{P})$ could be computed in a similar way, for which we present Algorithm~\ref{alg:ul}. 
As a conclusion, Algorithm~\ref{alg:ps} is computationally efficient. 

\begin{algorithm}
\caption{$\mathtt{Policy\,Search}$}
\begin{algorithmic}\label{alg:ps}
\STATE{$\textbf{Input:}$ reward $u$, $u'$, confidence region $\mathcal{P}=\otimes_{h,s,a}\mathcal{P}_{h,s,a}$;}
\STATE{\textbf{Initialization:} threshold $\epsilon=  \frac{1}{(SAHK)^10}$,  $b\leftarrow \max_{\pi}L^{\pi}(u,\mathcal{P})$, $a\leftarrow \max_{\pi}U^{\pi}(u+\textbf{1}_{z},\mathcal{P})$; $\eta_0 \leftarrow (a-b)/2$;}
\FOR{$i=0,1,2,\ldots,$}
\STATE{$\{\pi^{(i)},P^{(i)}\}\leftarrow \mathtt{EVI}(u+\textbf{1}_{z}+\eta_{i}u',\mathcal{P})$;}

\IF{$\frac{1}{\epsilon}\leq \eta_i<\frac{2}{\epsilon}$;}
\STATE{\textbf{Return:} $\pi^{(i)}$;}
\ELSIF{ $W^{\pi^{(i)}}(u,P^{(i)})\leq  b$ \label{line:mark1}}
\STATE{$\xi = \frac{b- W^{\pi^{(i)} } (u,P^{(i)})} {     W^{\pi^{(i-1)}}(u,P^{(i-1)})  - W^{\pi^{(i)} } (u,P^{(i)})  } $;}
\STATE{$(\check{\pi},\check{P})\leftarrow \mathtt{Sum}( \{\xi, \pi^{(i-1)} ,P^{(i-1)}\}, \{1-\xi, \pi^{(i)} ,P^{(i)}\}  )$}
\STATE{\textbf{Return:} $\check{\pi}$ ;}
\ELSE
\STATE{$\eta_{i+1}=2\eta_{i}$;}
\ENDIF
\ENDFOR

\end{algorithmic}
\end{algorithm}

\begin{algorithm}
\caption{$\mathtt{Extended\, Value\, Iteration}$ $(\mathtt{EVI})$}
\begin{algorithmic}\label{alg:evi}
\STATE{ \textbf{Input:} reward function $u$, confidence region $\mathcal{P}=\otimes_{h,s,a}\mathcal{P}_{h,s,a}$ }
\STATE{\textbf{Initialize:}  $Q_h(s,a)\leftarrow 0, V_h(s)\leftarrow 0,\forall (h,s,a)\in [H+1]\times \mathcal{S}\times \mathcal{A}$}
\FOR{$h = H,H-1,\ldots,1$}
\STATE{ $Q_{h}(s,a)\leftarrow \max_{q\in \mathcal{P}_{h,s,a}}\left( u(s,a)+qV_{h+1}\right)$, $\forall (s,a)\in \mathcal{S}\times \mathcal{A}$; }
\STATE{$p_{h,s,a}\leftarrow \arg\max_{q\in\mathcal{P}_{h,s,a}}\left(u(s,a)+qV_{h+1}\right)$;}
\STATE{$V_h(s)\leftarrow \max_{a}Q_{h}(s,a)$, $\forall s\in \mathcal{S}$;}
\STATE{$\pi_h(a|s)\leftarrow \mathbb{I}[a=\arg\max_{a'}Q_h(s,a') ],\forall (s,a)$;}
\ENDFOR
\STATE{\textbf{Return:} $\{\pi,p\}$.}
\end{algorithmic}
\end{algorithm}

\begin{algorithm}
\caption{$\mathtt{Upper\& Lower\,Confidence\,Bound}$}
\begin{algorithmic}\label{alg:ul}
\STATE {\textbf{Input:} reward function $u$, confidence region $\mathcal{P}=\otimes_{h,s,a}\mathcal{P}_{h,s,a}$;}
\STATE {\textbf{Initialize}: $\overline{Q}_h(s,a),\overline{V}_h(s),\underline{Q}_{h}(s,a),\underline{V}_h(s)\leftarrow 0$, $\forall (h,s,a)\in [H+1]\times \mathcal{S}\times \mathcal{A}$;}
\FOR{$h=H,H-1,\ldots,1$}
\STATE{$\overline{Q}_h(s,a)\leftarrow \max_{q\in \mathcal{P}_{h,s,a}}(u(s,a)+q\overline{V}_{h+1} )$, $\forall (s,a)\in \mathcal{S}\times \mathcal{A}$;}
\STATE{$\overline{V}_h(s)\leftarrow \max_{a}\overline{Q}_h(s,a)$, $\forall s\in \mathcal{S}$;}
\STATE{$\underline{Q}_h(s,a)\leftarrow \min_{q\in \mathcal{P}_{h,s,a}}(u(s,a)+q\underline{V}_{h+1} )$, $\forall (s,a)\in \mathcal{S}\times \mathcal{A}$;}
\STATE{$\underline{V}_h(s)\leftarrow \max_{a}\underline{Q}_h(s,a)$, $\forall s\in \mathcal{S}$;}
\ENDFOR
\STATE{\textbf{Return:} $\max_{\pi}U^{\pi}(u,\mathcal{P}):= \overline{V}_1(s_1)$, $\max_{\pi}L^{\pi}(u,\mathcal{P}):= \underline{V}_1(s_1)$;}
\end{algorithmic}
\end{algorithm}

\subsection{Theoretical Results and Proofs for Algorithm~\ref{alg:ps}}\label{app:cei}
\begin{lemma}\label{lemma:cei11}
	Let  $u,u'$ be  two reward functions and $\mathcal{P}$ be  a set of transition models.
		Assume $\mathcal{P} = \otimes_{h,s,a}\mathcal{P}_{h,s,a}$ is \emph{tight} w.r.t. a transition model  $P$. 
	Then by Algorithm~\ref{alg:ps} we can find $\pi$ such that
				\begin{align}
				W^{\pi}(u',P) \geq \frac{1}{18}\max_{\pi'\in \Pi(u,{\mathcal{P}}),}W^{\pi'}(u',P) -\frac{2}{9}\epsilon \nonumber
				\end{align}
				 in time $O(S^4AHM^3\log(SAHK)\log(SAHK/(a-b)))$, where $a = \max_{\pi}U^{\pi}(u+\textbf{1}_z,\mathcal{P})$ and $b=\max_{\pi}L^{\pi}(u,\mathcal{P})$.
\end{lemma}

\begin{proof}

Let $\tilde{u} = u+\textbf{1}_{z}$.
For any $\eta\geq 0$, we define $(\pi^{\eta}, p^{\eta}) $ be the policy-transition pair such that
\begin{align}
(\pi^{\eta},P^{\eta})= \arg\max_{\pi, p\in \mathcal{P} } W^{\pi}(\tilde{u}+\eta u',p).\nonumber
\end{align}
By Lemma~\ref{lemma:evi}, 
 with Algorithm~\ref{alg:ul}, we can compute $a$ and $b$ within time $\tilde{O}\left(S^4AHM^3\log(SAHK) \right)$. In the same way,
 with Algorithm~\ref{alg:evi} we can find $(\pi^{\eta}, p^{\eta}) $
within time $\tilde{O}\left(S^4AHM^3\log(SAHK) \right)$ for any $\eta>0$.  Note that in Algorithm~\ref{alg:ps}, the value of $i$ is at most $\log(1/(\eta_0\epsilon))=O(\log(\frac{1}{\epsilon(a-b)}))=O(\log(SAHK))$. As a result, the computational cost is at most $O(S^4AHM^3\log(SAHK)\log(SAHK/(a-b)))$.

We continue with an useful property of $(\pi^{\eta},P^{\eta})$.
\begin{lemma}\label{lemma:etaaa}
        Let $0< \eta< \eta'$ be fixed. Let $(\pi^\eta,p^{\eta})$, $(\pi^{\eta'},P^{\eta'})$ be such that
        \begin{align}
             & (\pi^{\eta},P^{\eta})= \arg\max_{\pi, p\in \mathcal{P} } W^{\pi}(\tilde{u}+\eta u',p) \nonumber
             \\ & (\pi^{\eta'},P^{\eta'})= \arg\max_{\pi, p\in \mathcal{P} } W^{\pi}(\tilde{u}+\eta' u',p).\nonumber
        \end{align}
        Then we have that
        \begin{align}
            W^{\pi^{\eta}}(\tilde{u},P^{\eta}) \geq W^{\pi^{\eta'}}(\tilde{u},P^{\eta'}).\nonumber
        \end{align}
\end{lemma}

\begin{proof}
 Let $x_1 = W^{\pi^{\eta}}(\tilde{u},P^{\eta})$, $x_2 =W^{\pi^{\eta'}}(\tilde{u},P^{\eta'}) $, $y_1 =  W^{\pi^{\eta}}(u',P^{\eta})$ and $y_2 =W^{\pi^{\eta'}}(u',P^{\eta'}) $. It suffices to show that $x_1\geq x_2$. By the optimality of $(\pi^{\eta},P^{\eta})$ and $(\pi^{\eta'},P^{\eta'})$,   we have that
 \begin{align}
 & x_1+\eta y_1 \geq x_2+\eta y_2;\nonumber 
 \\ & x_2+\eta'y_2 \geq x_1 +\eta' y_1.\nonumber 
 \end{align}
 
 If $x_1<x_2$, then we have that $y_1>y_2$. It then follows that  $x_2+\eta'y_2= x_2 + \eta y_2 + (\eta'-\eta)y_2< x_1+\eta y_1 + (\eta'-\eta) y_1 = x_1+\eta' y_1$, which leads to contradiction. 
\end{proof}

In Algorithm~\ref{alg:ps}, there are two breaking conditions. 

\paragraph{Case 1}
Recall that $\{\pi^{(i)},P^{(i)}\}=\arg\max_{\pi,p\in \mathcal{P}}W^{\pi}(u+\textbf{1}_{z}+\eta_i\mu',p) = \arg\max_{\pi,p\in \mathcal{P}}W^{\pi}(\tilde{u}+\eta_i\mu',p)$
In the first case, we end with obtaining some $i$ satisfying that
\begin{align}
    W^{\pi^{(i)}}(\tilde{u},P^{(i)})\leq b.\nonumber
\end{align}

Because $W^{\pi^{(0)}}(\tilde{u},P^{(0)})\geq a-\eta_0>b$, it holds that $\eta_{i}>\eta_0$ for any $i\geq 1$. By Lemma~\ref{lemma:etaaa} and the stopping condition, we have that $  W^{\pi^{(i-1)}}(\tilde{u},P^{(i-1)})\geq  b$.  By Lemma \ref{lemma:stp}, we can find a policy $\check{\pi}$ and $\check{P}\in \mathcal{P}$ such that 
\begin{align}
W^{\check{\pi}}(v,\check{P})  = \xi  W^{ {\pi}^{(i)}}(v, {P}^{(i)}) +(1-\xi) W^{ {\pi}^{(i-1)} }(v, {P}^{(i-1)})  \label{eq:add2}
\end{align}
for any reward function $v$.

Noting that $\xi = \frac{b- W^{\pi^{(i)} }(\tilde{u},P^{(i)}) } {     W^{\pi^{(i-1)}}(\tilde{u},P^{(i-1)})  - W^{\pi^{(i)} }(\tilde{u},P^{(i)})   }$, we have that $U^{\check{\pi}}(u,\mathcal{P})\geq  W^{\check{\pi}}(\tilde{u},\check{P}) = \xi  W^{\pi^{(i)}}(\tilde{u},P^{(i)}) +(1-\xi) W^{\pi^{(i-1)}}(\tilde{u},P^{(i-1)}) = b$,
which implies that $\check{\pi} \in \Pi(u,\mathcal{P})$. 

Note that $ W^{\pi}(v,p)$ is linear in $v$ for fixed $\pi$ and $p$. For any policy $\pi\in \Pi(r,\mathcal{P})$ and $p'\in \mathcal{P}$, we have that
\begin{align}
&  W^{\pi}(\tilde{u},p')+ \eta_{i}  W^{\pi} (u',p')\leq  W^{ {\pi}^{(i)}}(\tilde{u},  {p}^{(i)}) +\eta_i  W^{  {\pi}^{(i)} }(u', {P}^{(i)}),\label{eq:cei1}
\\ &   W^{\pi }(\tilde{u},p')+ \eta_{i-1} W^{\pi }(u',p') \leq  W^{ {\pi}^{(i-1)}}(\tilde{u}, {P}^{(i-1)}) +\eta_{i-1}  W^{  {\pi}^{(i-1)} }(u',  {P}^{(i-1)}) .\label{eq:cei2}
\end{align}
It then follows that
\begin{align}
&   W^{\pi}(\tilde{u},p')+ \eta_{i-1}  W^{\pi }(u',p') \nonumber
\\ & \leq \xi \left(  W^{\pi^{(i)}}(\tilde{u},P^{(i)}) +\eta_i  W^{ \pi^{(i)} }(u',P^{(i)}) \right) +(1-\xi)\left(    W^{\pi^{(i-1)}}(\tilde{u}, P^{(i-1)}) +\eta_{i-1}  W^{ \pi^{(i-1)} }(u',P^{(i-1)})\right) \nonumber
\\ &  \leq b + \eta_{i} W^{\check{\pi}}(u',\check{P}).\label{eq:cei3}
\end{align}
For any  $\pi\in \Pi(u,\mathcal{P})$, there exists $p'\in \Pi(u,\mathcal{P})$ such that $ W^{\pi}(\tilde{u},p')\geq b$.
By  \eqref{eq:cei3} and noting that $\eta_{i}= 2\eta_{i-1}$, we have
\begin{align}
W^{\pi }(u', p')  \leq \frac{\eta_{i}}{\eta_{i-1}} W^{\check{\pi}}(u',\check{P}) \leq 2 W^{\check{\pi}}(u',\check{P}).
\end{align}
On the other hand, by Lemma~\ref{lemma:tight}, for any $\pi$ it holds that
\begin{align}
 W^{\pi }(u', p )\leq 3W^{\pi }(u', p')\leq 9 W^{\pi }(u', p ) ,
\end{align}
for any $p'\in \bar{\mathcal{P}}$,
which implies that
\begin{align}
    W^{\check{\pi}}(u',p) \geq \frac{1}{6}\max_{\pi\in \Pi(u,\mathcal{P})}W^{\pi}(u', p) .\nonumber
\end{align}

\paragraph{Case 2}
In the second case, we end with some $i$ such that $\frac{1}{\epsilon}\leq \eta_i<\frac{2}{\epsilon} $.

In this case, because $W^{ {\pi}^{(i)}}(\tilde{u}, {P}^{(i)})\geq b$, we have that $\pi^{(i)}\in \Pi(u,\mathcal{P})$ . 
For any $\pi\in \Pi(u,\mathcal{P})$ such that
\begin{align}
    W^{\pi}(u',p) \geq 18 W^{\pi^{(i)}}(u',p),\label{eq:cei14}
\end{align}
by the \emph{tightness} of $\mathcal{P}$ (w.r.t. $p$)
it holds that
\begin{align}
    \eta_{i} W^{\pi}(u',p')& \geq \frac{\eta_i}{3}W^{\pi}(u',p) \geq 6\eta_i W^{\pi^{(i)}}(u',p)  \geq            2\eta_i W^{\pi^{(i)}}(u',P^{(i)})\label{eq:cei++}
\end{align}
 for any $p'\in \mathcal{P}$.
On the other hand, by optimality of $(\pi^{(i)}, P^{(i)})$, we have that
\begin{align}
     \eta_i W^{\pi}(u',p') \leq W^{\pi^{(i)}}(\tilde{u},P^{(i)})+ \eta_i  W^{\pi^{(i)}}(u',P^{(i)}).\label{eq:cei12}
\end{align}
Combine \eqref{eq:cei++} with \eqref{eq:cei12}, we have that
\begin{align}
    \eta_i  W^{\pi^{(i)}}(u',P^{(i)}) \leq  W^{\pi^{(i)}}(\tilde{u},P^{(i)}) \leq 2.\label{eq:cei13}
\end{align}
Combining \eqref{eq:cei12} with \eqref{eq:cei13}, for any $p'\in \mathcal{P}$, using the optimality of $(\pi^{(i)},P^{(i)})$ and \eqref{eq:cei13}, we have that
\begin{align}
     \eta_i W^{\pi}(u',p') \leq W^{\pi^{(i)}}(u,P^{(i)})+ \eta_i  W^{\pi^{(i)}}(u',P^{(i)}) \leq 4.
\end{align}
It then follows $W^{\pi}(u',p) \leq  4\epsilon$. Therefore, for any $\pi\in \Pi(u,\mathcal{P})$, it holds either $ W^{\pi}(u', p) \leq 18 W^{\pi^{(i)}}(u',p)$ or $W^{\pi}(u',p)  \leq 4\epsilon$. We then have that
\begin{align}
    W^{\pi^{(i)}}(u',p) \geq \frac{1}{18} \max_{\pi\in \Pi(u,\mathcal{P})} W^{\pi}(u', p)   - \frac{2}{9}\epsilon.\label{eq:cei17}
\end{align}

The proof is completed.

\end{proof}

\begin{lemma}\label{lemma:evi}
The computational cost of Algorithm~\ref{alg:evi} and Algorithm~\ref{alg:ul} is bounded by $O(S^3AHM^3\log(SAKH))$.
\end{lemma}
\begin{proof}
To implement the two algorithm, we need to solve $SAH$ linear optimization problem, which has the form $\max_{q\in \mathcal{P}_{h,s,a}} (r+qv)$ or $\min_{q\in \mathcal{P}_{h,s,a}} (r+qv)$. Note that $\mathcal{P}_{h,s,a}$ has the form $\{p\in \Delta^{\mathcal{S}}: a_i^{\top}(p-p')\leq b_i, i\geq 1\}$, and the number of linear constraints is increased for at most $O(S)$ in each batch. As a result, the total number of linear constraints in $\mathcal{P}_{h,s,a}$ is bounded by $O(SM)$. By the results in \cite{cohen2021solving}, the time cost to solve the linear program problem above is bounded by $O(S^3M^3\log(SAHK))$. Therefore, the total computational cost is bounded by $O(S^3AHM^3\log(SAKH))$.  
\end{proof}

	\section{Proof  of Theorem~\ref{thm:main}}\label{sec:proof}
	\paragraph{Additional Notations}
	In this section, we use
	$N^m_h(s,a,s')$ to denote the visit count of $(s,a,h,s')$ after the $m$-th batch. We also define $N^m_h(s,a)=\max\{\sum_{s'}N^m_h(s,a,s'),1\}$. We use $\{\check{N}_h^m(s,a,s')\}$ to denote the counts of the $m$-th batch. Similarly we define $\check{N}_h^m(s,a) = \max\{ \sum_{s'}\check{N}_h^m(s,a,s'),1 \}$. Let $W^*$ be the \emph{known} set after the first two stages. Let $\hat{P}^m_{h,s,a,s'}= \frac{N^m_h(s,a,s')}{N^m_h(s,a)}$ be the empirical transition model for $1\leq m \leq 2H+M$. For $2H+1\leq m \leq 2H+M$, define $\{\check{P}^m_{h,s,a}\}$ be the clipped transition model, i.e., $\{\check{P}^m_{h,s,a} \}_{h,s,a}= \mathtt{clip}\left(\left\{\left[\frac{\check{N}_h^m(s,a,s')}{\check{N}^m_{h}(s,a)}\right]_{s'\in \mathcal{S}}\right\}_{h,s,a} ,\mathcal{W}^*\right)$.

Note that the $m$-batch in Algorithm~\ref{alg:policy_elimination} indicates the $2H+m$-th batch in the main algorithm. To align the indices, with a slight abuse of notations we use $\mathcal{P}^m$ and $v^m$ to denote respectively the value of $\mathcal{P}^{m-2H}$ and $v^{m-2H}$ in  Algorithm~\ref{alg:policy_elimination} for $m\geq 2H$. 
	
\begin{table}[!h]\label{table1}
	\centering
	\caption{Explanation of the notations}
	\begin{tabular}{|l|l|}
			\hline
		   $W^{\pi}(u,p)$  & the general value function: $W^{\pi}(u,p)=\mathbb{E}_{p,\pi,s_1\sim \mu_1}[\sum_{h=1}^H u_h(s_h,a_h)]$\\
		   	\hline
		   $U^{\pi}(u,\mathcal{P}) $  & the upper confidence bound  w.r.t. policy $\pi$, reward $u$ and confidence region $\mathcal{P}$ ;\\
		    \hline
		    $L^{\pi}(u,\mathcal{P}) $  & the lower confidence bound  w.r.t. policy $\pi$, reward $u$ and confidence region $\mathcal{P}$ ;\\
		\hline
		   $N^m_h(s,a,s')$  & the visit count of $(s,a,h,s')$ after the $m$-th batch\\
  \hline
   $N_h^m(s,a)$ & $N_h^m(s,a)=\max\{  \sum_{s'}N^m_h(s,a,s') ,1\}$;\\
   \hline
    $\check{N}_h^m(s,a,s')$ & the count of $\big(h,s,a,s'\big)$ in the $m$-th batch;\\   \hline
     $\check{N}_h^m(s,a)$ &  $\check{N}_h^m(s,a) = \max\left\{ \sum_{s'}\check{N}_h^m(s,a,s'),1 \right\}$\\  \hline
     $W^*$ & the \emph{known} set after the first two stages\\  
       \hline
     $\hat{P}^m_{h,s,a,s'}$&   $\hat{P}^m_{h,s,a,s'}= \frac{N^m_h(s,a,s')}{N^m_h(s,a)}$, the empirical transition probability;\\ 
     \hline
      $\check{P}^m_{h,s,a}$ & $\{\check{P}^m_{h,s,a} \}_{h,s,a}= \mathtt{clip}\left(\left\{\left[\frac{\check{N}_h^m(s,a,s')}{\check{N}^m_{h}(s,a)}\right]_{s'\in \mathcal{S}}\right\}_{h,s,a} ,\mathcal{W}^*\right)$;\\  
     \hline
      $\bar{P}$ & $\bar{P}=\mathtt{clip}\Big(P,W^*\Big)$, the clipped true transition model;\\
      \hline
      $\mathcal{P}^m$ &  the confidence region after the $m$-th batch;\\
      \hline
      $\{v_h^m(s)\}$ &  the extended optimal value function after the $m$-th batch;\\
      \hline 
      $V^* \Big(\bar{V}^*\Big)$  & the optimal value function for the (clipped) true transition model;\\
      \hline
      $\alpha(n,n')$ & $\alpha(n,n')=\sqrt{\frac{4n'\iota}{n^2}}+\frac{5\iota}{n}$;\\
      \hline 
      $\alpha^*(n,p,v)$ & $\alpha^*(n,p,v) = 5\sqrt{\frac{\mathbb{V}(p,v)\iota}{n}}+\frac{3\iota}{n}$;\\
      \hline 
	\end{tabular}	
\end{table}


\paragraph{The good event}
 For $1\leq m \leq 2H+M$,
		define $\mathcal{G}_{h,s,a,s'}^{m}$ be the event where it holds
\begin{align}		
& 	\left|\hat{P}^{m}_{h,s,a,s'}-P_{h,s,a,s'}\right| \leq \beta^{m}_{h,s,a,s'}:=	 \min\left\{ \sqrt{\frac{2P_{h,s,a,s'}\iota}{N^m_h(s,a)}}+ \frac{\iota}{3\cdot N^m_h(s,a)},\sqrt{\frac{4\check{P}^{m}_{h,s,a,s'}\iota }{N^m_h(s,a)}} +\frac{5\iota}{ N^m_h(s,a)} \right\}. \label{eq:conf1}
	\end{align}
	By Lemma~\ref{lemma:bennet} and Bernstein inequality, we have that $\mathbb{P}[\mathcal{G}_{h,s,a,s'}^{m}]\geq 1-2\delta$ .

For $1\leq m \leq 2H$, we set $\check{\mathcal{G}}_{h,s,a}^m$ to be the whole event. For $2H+1\leq m \leq M$, we
define $\check{\mathcal{G}}_{h,s,a}^m$ be the event where it holds 
\begin{align}
    \\ & \left| (\check{P}_{h,s,a}-P)v^{m-1} \right|\leq \lambda^m_{h,s,a}: = \min\left\{ 5\sqrt{\frac{\mathbb{V}(\check{P}_{h,s,a}^m,v^{m-1})\iota}{\check{N}_h^m(s,a)}} \right\}\label{eq:conf2}
    \\ &  \left| (\check{P}_{h,s,a}-P)\bar{V}^*\right|\leq \lambda^{m,*}_{h,s,a}: = \min\left\{ 5\sqrt{\frac{\mathbb{V}(\check{P}_{h,s,a}^m,\bar{V}^*)\iota}{\check{N}_h^m(s,a)}} \right\}.\label{eq:conf3}
\end{align}
Noting that $\check{P}_{h,s,a}$ is independent with both $\bar{V}^*$ and $v^{m-1}$, by Bernstein's inequality, we have that $\mathbb{P}[\check{\mathcal{G}}_{h,s,a,s'}^{m}]\geq 1-4\delta$

	The good event $\mathcal{G}$ is defined as $ 
	\mathcal{G} = \bigcap_{h,s,a,s'}\bigcap_{m=1}^{M}\left(\mathcal{G}^{m}_{h,s,a,s'} \cap  \check{\mathcal{G}}_{h,s,a}^m \right)$
	Then $\mathbb{P}[\mathcal{G}]\geq 1-6S^2AHM\delta$.
	Throughout the analysis, we always assume  $\mathcal{G}$ holds.

\begin{lemma}\label{lemma:ge1}
        Conditioned on $\mathcal{G}$, we have $\bar{P}\in \mathcal{P}^m$ for $2H\leq m \leq 2H+M$.
\end{lemma}
	
Noting that the batch complexity is bounded by $2H +M=O(H+\log_2\log_2(K))$, it suffices to prove the regret bound.	
We start with counting the regret in the first two stages. 
 The regret in the first batch is bounded by $O(H^2k_1)$ trivially. As for the second batch, we have that
\begin{lemma}\label{lemma:b2}
        Conditioned on $\mathcal{G}$, with probability $1-4SAH\delta$ the regret bound in the second batch is bounded by $O\left(\frac{k_2\sqrt{S^4A^3H^8\iota}}{\sqrt{k_1}}+\frac{k_2S^3A^3H^4\iota}{k_1}\right)$.
\end{lemma}

To count the regret in the third stage, we first show that the difference between the clipped model and the original model could be ignored.

\begin{lemma}\label{lemma:raw22}
        Conditioned on $\mathcal{G}$, with probability $1-4S^2AH^2\delta$, for any optimal policy $\pi^*$, it holds that $\mathrm{Pr}_{\pi^*}[\exists h\in [H], (h,s_h,a_h,s_{h+1})\notin \mathcal{W}^*]\leq O\left( \frac{ S^3A^2H^3\iota }{k_2}\right)$
\end{lemma}

Based on Lemma~\ref{lemma:raw22}, we further have that
\begin{lemma}\label{lemma:key1}
      Recall that $\bar{V}^*$ be the optimal value function with respect to the transition model $\bar{P}$ and reward function $r$. It then holds that $\bar{V}^*_1(s_1)\leq V^*_1(s_1)\leq \bar{V}^*_1(s_1)+ O\left( \frac{ S^3A^2H^4\iota }{k_2}\right) $.
\end{lemma}
\begin{proof}
The left side is obvious since the reward at $z$ is always 0. On the other hand, letting $\pi^*$ be an optimal policy and $E$ be the event where $\exists h\in [H], (h,s_h,a_h,s_{h+1})\notin \mathcal{W}^*$. Then we have that 
\begin{align}
V_1^{\pi^*}(s_1) & \leq \mathbb{E}_{\pi^*}\left[\left(\sum_{h=1}^H r_h(s_h,a_h)\right) \mathbb{I}[E] \right]+H\mathrm{Pr}_{\pi^*}[E]\nonumber
\\ & \leq \mathbb{E}_{\pi^*}\left[\sum_{h=1}^H r_h(s_h,a_h) \mathbb{I}[\forall h'<h, (h',s_{h'},a_{h'},s_{h‘+1})\in \mathcal{W}^*] \right]+ O\left( \frac{ S^3A^2H^4\iota }{k_2}\right)\nonumber
\\ & = \bar{V}^{\pi^*}_1(s_1)+ O\left( \frac{ S^3A^2H^4\iota }{k_2}\right).\nonumber
\end{align} 
\end{proof}

Recall that $\mathrm{gap}^{m+1} := \max_{\pi\in \Pi(r,\mathcal{P}^{m})}  (U^{\pi}(\mathcal{P}^m)- L^{\pi}(\mathcal{P}^m))$. 
For $m\geq 2H+1$, we have that
\begin{lemma}\label{lemma:gapbound}
        Conditioned on $\mathcal{G}$, with probability $1-4SAHKM\delta$, it holds that
        \begin{align}
        &\mathrm{gap}^{m+1} \nonumber
        \\ & \leq  O\left(\sqrt{\frac{SAH^3\ln(K)\iota^2}{K_{m-2H}}}+  \frac{SAH^2\ln(K)\iota}{K_{m-2H}}+\sqrt{\frac{S^\frac{11}{2}A^4H^7\ln(K)\iota^{\frac{5}{2}}}{ K_{m-2H}k_1 }} +\sqrt{\frac{S^4A^{\frac{5}{2}}H^4\ln(K)\iota^\frac{3}{2}}{ K_{m-2H}\sqrt{k_1} }} \right).\label{eq:gapmb}
        \end{align}
\end{lemma}

By Lemma~\ref{lemma:ge1}, \ref{lemma:key1} and \ref{lemma:gapbound}, for any $2H\leq m \leq 2H+K$ and any $\pi\in \Pi(\mathcal{P}^{m})$, we have that
$$V^{\pi}_{1}(s_1)\geq L^{\pi}(\mathcal{P}^{m})\geq U^{\pi}(\mathcal{P}^{m})-\mathrm{gap}^{m+1}\geq \bar{V}_1^*(s_1)-\mathrm{gap}^{m+1}- O\left(\frac{S^3A^2H^4\iota}{k_2} \right).$$

Recall that $k_1 = 144\sqrt{SAK\iota/H}$, $k_2 =288S^3A^2H^4\sqrt{K\iota} $ and $K_m = \left\lceil K^{1-\frac{1}{2^m}}\right \rceil$ for $1\leq m \leq M$. It then holds that $\frac{K_{m-2H+1}}{\sqrt{K_{m-2H}}} = \sqrt{K}$ for any $2H+1\leq m \leq 2H+K$.
Noting that the regret in the $m+1$-th batch is bounded by $K_{m+1-2H}\cdot \mathrm{gap}^{m+1}$,  
and  the regret in the $2H+1$-th batch is bounded by $K_1= O(\sqrt{K})$, the total regret is bounded by 
\begin{align}
    \mathrm{Regret}(K) =M \cdot O\left(     \sqrt{SAH^3K\ln(K)\iota^2} +  S^{\frac{15}{4}}A^{\frac{9}{8}}H^{\frac{17}{8}}\iota^{\frac{5}{8}}K^{\frac{3}{8}}+   S^{\frac{19}{4}}A^{\frac{13}{4}}H^{\frac{33}{4}}\ln(K)\iota K^{\frac{1}{4}}+ S^{\frac{11}{2}}A^{\frac{9}{2}}H^{\frac{17}{2}}\iota \right). \nonumber
\end{align}
By replacing $\delta$ by $\frac{\delta}{20S^2AHK}$, we get the desired regret bound.

Below we analyze the computational cost of Algorithm~\ref{alg:main}.	
By Lemma~\ref{lemma:stp}   the computational costs of $\mathtt{Sum}$ is  $O(nS^3A^2H^2)$ , where $n$ is the number of inputs for $\mathtt{Sum}$.

Below we analyze the computational cost of $\mathtt{PolicySearch}$. By Lemma~\ref{lemma:cei11}, for input $(u,u',\mathcal{P})$, the computational cost of $\mathtt{PolicySearch}$ is bounded by $O(S^4AHM^3\log(SAHK)\log(SAHK/(a-b)))$ with $a = \max_{\pi}U^{\pi}(u+\textbf{1}_z,\mathcal{P})$ and $b=\max_{\pi}L^{\pi}(u,\mathcal{P})$.

 In the first stage, we invoke $\mathtt{PolicySearch}$ with $u = 0$, which implies $b= 0$ and $W^{\pi}(u,p)=0$ for any $\pi$ and $p\in \mathcal{P}$. Then the condition in Line~\ref{line:mark1} Algorithm~\ref{alg:ps} is satisfied and the loop would break. Therefore, by Lemma~\ref{lemma:evi}, the computational cost of $\mathtt{PolicySearch}$ in the first stage is bounded by  $O(S^4AHM^3\log(SAKH))$.

In the second and the third stage, we invoke $\mathtt{PolicySearch}$ with $u = r$. In this case, if $a-b\leq 1/K$, then
we can learn an $1/K$-optimal policy by solving $\pi' = \arg\max_{\pi}L^{\pi}(r,\mathcal{P})$. Then we can simply run this policy in the left episodes. Without loss of generality, we then assume that $a-b >1/K$, which implies the time cost of $\mathtt{PolicySearch}$ is bounded by  
 $O(S^4AHM^3\log^2(SAKH))$.

Now we count the number of callings to $\mathtt{Sum}$ and $\mathtt{PolicySearch}$. In the first and second stage, $\mathtt{Sum}$ is called for $2H$ times with $n=SAH$ inputs, and $\mathtt{PolicySearch}$ is called for $2H$ times. In the third stage,  $\mathtt{Sum}$ is called for $M$ times with $n=K^3$ inputs, and $\mathtt{PolicySearch}$ is called  for $K^3M$ times. So the total time cost due to  $\mathtt{Sum}$ and $\mathtt{PolicySearch}$ is bounded by $\tilde{O}(S^4AHK^3+S^3A^2H^2K^3)$. On the other hand, to compute $\{v^m_h(s)\}_{h\in [H],s\in \mathcal{S}}$ in Line~\ref{line:v} Algorithm~\ref{alg:ps}, we need to invoke $\mathtt{EVI}$ (see Algorithm~\ref{alg:evi}) for $M$ times, which needs additional $O(S^4AHM^4\log(SAHK))$ time by Lemma~\ref{lemma:evi}. Finally, to observe the samples and compute the confidence region, we need $O(S^2AHK)$ time.

Putting all together, the computational cost of Algorithm~\ref{alg:main} is bounded by $\tilde{O}(S^4AHK^3+S^3A^2H^2K^3)$. The proof is completed.

\subsection{Proof of Lemma~\ref{lemma:ge1}}
\textbf{Lemma~\ref{lemma:ge1} (restated)}
\emph{
        Conditioned on $\mathcal{G}$, we have $\bar{P}\in \mathcal{P}^m$ for $2H\leq m \leq 2H+M$.}
        
\begin{proof}

with a slight abuse of notation, we use $v^{m}$ to denote the value of $v^{m-2H}$ in Algorithm~\ref{alg:policy_elimination}.

Recall the definition of $\mathcal{P}^m$.
It suffices to show that  $\bar{P}\in \mathtt{CR}^*(\mathcal{D}^m,\bar{\mathcal{D}^m,W^*,\{ v_{h}^{m-1}(s) \}_{(h,s)}})$ for each $m\geq 2H$. 

Note that after the $m$-th batch $\hat{p}_{h,s,a,s'}=\hat{P}^m_{h,s,a,s'}$ and $\check{p}_{h,s,a}=\check{P}^m_{h,s,a}$.
By the definition of $\mathcal{G}$, and recalling the definition of $\beta_{h,s,a,s'}^m$ and $\lambda_{h,s,a}^m$ in \eqref{eq:conf1} and \eqref{eq:conf2}, we have that
\begin{align}
 &     \left| \bar{P}_{h,s,a,s'} -\hat{P}^m_{h,s,a,s'} \right|\leq \beta_{h,s,a,s'}^m \leq \alpha(N^m_h(s,a),N^m_h(s,a,s') )\nonumber
 \\ &  \left|(\bar{P}_{h,s,a}- \check{P}^m_{h,s,a} )v^{m-1} \right|\leq \lambda_{h,s,a}^m \leq \alpha^*(\check{N}_h^m(s,a),\check{P}_{h,s,a}^m,v^{m-1}  ).\nonumber
\end{align}
The proof is completed.
\end{proof}

\subsection{Proof of Lemma~\ref{lemma:b2}}

\textbf{Lemma~\ref{lemma:b2} (restated)}\emph{
      Conditioned on $\mathcal{G}$, with probability $1-4SAH\delta$ the regret bound in the second stage is bounded by $O\left(\frac{k_2\sqrt{S^4A^3H^8\iota}}{\sqrt{k_1}}+\frac{k_2S^3A^3H^4\iota}{k_1}\right)$.}
      
      \begin{proof}
      Let $\mathcal{D}^1$ and $\mathcal{D}^2$ be respectively the dataset after the first and second stage. Let $\{\bar{N}^1_{h}(s,a,s')\}$ and $\{\bar{N}^2_{h}(s,a,s')\}$ be the corresponding counts. Let $\bar{\mathcal{W}}^1$ and $\bar{\mathcal{W}}^2$ be the corresponding \emph{known} set. Note that $\mathcal{W}^* = \bar{\mathcal{W}}^2$.
       By Lemma~\ref{lemma:raw++}, with probability $1-8S^2AH^2\delta$, it holds that
       \begin{align}
            & \max_{\pi}\mathbb{P}_{\pi}\left[ \exists h\in [H], (h,s_h,a_h,s_{h+1})\notin \bar{\mathcal{W}}^1\right] \leq \frac{36C_1S^2A^2H^3\iota}{k_1}\nonumber
            \\ & \bar{N}^1_h(s,a)\geq \frac{ck}{27SA}\max_{\pi}W^{\pi}(\textbf{1}_{h,s,a},P)-4\iota-\frac{36C_1SAH^3\iota}{27}.\label{eq:xgx}
       \end{align}
      For any policy $\pi$ in $\Pi(\mathtt{CR}(\mathcal{D}^1))$, using policy difference lemma we have that
      \begin{align}
         &  U^{\pi}(\mathtt{CR}(\mathcal{D}^1))- L^{\pi}(\mathtt{CR}(\mathcal{D}^1)) \nonumber
         \\ & = U^{\pi}(\mathtt{CR}(\mathcal{D}^1))-W^{\pi}(r, \mathtt{clip}(P,\bar{\mathcal{W}}^1)) + W^{\pi}(r, \mathtt{clip}(P,\bar{\mathcal{W}}^1))-L^{\pi}(\mathtt{CR}(\mathcal{D}^1)) 
         \\ & \leq \max_{\pi}\mathbb{P}_{\pi}\left[ \exists h\in [H], (h,s_h,a_h,s_{h+1})\notin \bar{\mathcal{W}}^1\right] +O\left(\sum_{h,s,a}W^{\pi }(\textbf{1}_{h,s,a},\mathtt{clip}(P,\bar{\mathcal{W}}^1))\sqrt{\frac{S\iota }{\bar{N}^1_{h}(s,a)}}\cdot H\right)\nonumber
         \\ & \leq \frac{36C_1S^2A^2H^3\iota }{k_1} + O\left( \sum_{h,s,a} \left( \frac{SA(\bar{N}_h^1(s,a)+SAH^3\iota)}{k}\right) \sqrt{\frac{SH^2\iota}{\bar{N}_h^1(s,a)}}               \right)
         \\ & \leq \frac{36C_1S^2A^2H^3\iota}{k_1}+O\left(\ \sqrt{\frac{S^4A^3H^8\iota}{k_1}} +\frac{S^3A^3H^4\iota}{k_1}\right),\nonumber
      \end{align}
      where the third line is by \eqref{eq:xgx} and the last line is by Cauchy's inequality and the fact that $\bar{N}_h^1(s,a)\geq 1$. 
     Conditioned on $\mathcal{G}$, we have that = $\mathtt{clip}(P,\bar{\mathcal{W}}^1)\in \mathtt{CR}(\mathcal{D}^1)$. As a result, we have that $\max_{\pi}U^{\pi}(\mathtt{CR}(\mathcal{D}^1))\geq  V^*_1(s_1)-\frac{36C_1S^3A^2H^4\iota}{k_1}$. To conclude, the regret in the second stage is bounded by $O\left(\frac{k_2\sqrt{S^4A^3H^8\iota}}{\sqrt{k_1}}+\frac{k_2S^3A^3H^4\iota}{k_1}\right)$.
      \end{proof}

\subsection{Proof of Lemma~\ref{lemma:raw22}}
\textbf{Lemma~\ref{lemma:raw22} (restated)}\emph{
  Conditioned on $\mathcal{G}$, with probability $1-4S^2AH^2\delta$, for any optimal policy $\pi^*$, it holds that $\mathrm{Pr}_{\pi^*}[\exists h\in [H], (h,s_h,a_h,s_{h+1})\notin \mathcal{W}^*]\leq O\left( \frac{ S^3A^2H^3\iota }{k_2}\right)$.}
\begin{proof}
By Lemma~\ref{lemma:raw++}, with probability $1-4S^2AH^2\delta$, it holds that
\begin{align}
    \max_{\pi \in \Pi^*}\mathrm{Pr}_{\pi}\left[ \exists h\in [H] , (h,s_{h},a_{h},s_{h+1})\notin \mathcal{W}^* \right] \leq \frac{36C_1S^2A^2H^3\iota }{k_2}.\nonumber
\end{align}
The proof is completed.
\end{proof}

\subsection{Proof of Lemma~\ref{lemma:gapbound}}
\textbf{Lemma~\ref{lemma:gapbound} (restated)}\emph{
 Conditioned on $\mathcal{G}$, with probability $1-4SAHKM\delta$, it holds that}
        \begin{align}\mathrm{gap}^{m}  \leq  O\left(\sqrt{\frac{SAH^3\ln(K)\iota^2}{K_{m-2H}}}+  \frac{SAH^2\ln(K)\iota}{K_{m-2H}}+\sqrt{\frac{S^\frac{11}{2}A^4H^7\ln(K)\iota^{\frac{5}{2}}}{ K_{m-2H}k_1 }} +\sqrt{\frac{S^4A^{\frac{5}{2}}H^4\ln(K)\iota^\frac{3}{2}}{ K_{m-2H}\sqrt{k_1} }} \right)\nonumber\end{align}
for $2H+1\leq m \leq M$.
\begin{proof}

Let $m\in [2H+1,M]$ be fixed. Conditioned on $\mathcal{G}$, we have that for any $p\in \mathcal{P}^{m-1}$, for any $(h,s,a,s')\in \mathcal{W}^*$ it holds that
\begin{align}
    \left|\hat{P}^{m-1}_{h,s,a,s'}- \bar{P}_{h,s,a,s'}  \right| &\leq
    \sqrt{\frac{4\hat{P}^{m-1}_{h,s,a,s'}\iota}{N^{m-1}_h(s,a)}}+\frac{\iota}{3N^{m-1}_h(s,a)}\nonumber 
	    \\ & = \frac{1}{N^{m-1}_h(s,a)}\cdot (\sqrt{4N^{m-1}_h(s,a,s')\iota}+1/3) \nonumber 
	    \\ & \leq 3\hat{P}^{m-1}_{h,s,a,s'}\cdot \sqrt{\frac{\iota}{N^{m-1}_h(s,a,s')}} \nonumber 
	    \\ & \leq \frac{1}{3H}\hat{P}^{m-1}_{h,s,a,s'}.\nonumber
\end{align}
On the other hand, noting that for any $p\in \mathcal{P}^{m-1}$ and $(h,s,a,s')\in \mathcal{W}^{*}$, with similar computation it holds that
\begin{align}
    \left|p_{h,s,a,s'}-\bar{P}_{h,s,a,s'} \right|& \leq  \left|p_{h,s,a,s'}-\hat{P}^{m-1}_{h,s,a,s'} \right| +  \left|\hat{p}_{h,s,a,s'}-\bar{P}^{h'}_{h,s,a,s'} \right| \nonumber
    \\ & \leq \frac{1}{3H}\bar{P}_{h,s,a,s'}+\frac{1}{3H}\hat{P}^{m-1}_{h,s,a,s'}\nonumber
    \\ & \leq \left(\frac{2}{3H}+\frac{1}{9H^2}\right) \bar{P}_{h,s,a,s'}\nonumber 
\end{align}
Therefore $\mathcal{P}^{m-1}$ is \emph{tight} with respect to $\bar{P}$. Let $p^{m-1}\in \mathcal{P}^{m-1}$ be the value of $p$ in Line 26 Algorithm~\ref{alg:policy_elimination}. Let $r^{i,m-1}$ be the value of $r^i$ defined in Line 29  Algorithm~\ref{alg:policy_elimination}.
Let $\{\tilde{\pi}^{i,m-1}\}$ be the value of $\tilde{\pi}(i)$ in Line 30 Algorithm~\ref{alg:policy_elimination}. 

As a result, by Lemma~\ref{lemma:cei11}, Lemma~\ref{lemma:stp} and Lemma~\ref{lemma:tight}
\begin{align}
    & W^{\tilde{\pi}^{i,m-1}}(r^{i,m-1},\bar{P})\geq \frac{c}{9} \max_{\pi\in \Pi(\mathcal{P}^{m-1})}W^{\pi}(r^{i,m-1},\bar{P})\nonumber
    \\ & W^{\pi^{m}}(\textbf{1}_{h,s,a},\bar{P})\geq \frac{1}{9K^3}\sum_{i=1}^{K^3} W^{\tilde{\pi}^{i,m-1}}(\textbf{1}_{h,s,a},\bar{P}) ,\forall (h,s,a).\label{eq:designpi}
\end{align}
	Consequently, for any $\pi\in \Pi(\mathcal{P}^{m-1})$ and $(h,s,a)$, it holds that
	\begin{align}
	    W^{\pi}(r^{K^3+1,m-1},\bar{P})\leq & \frac{81}{cK^3}\sum_{i=1}^{K^3}W^{\tilde{\pi}^{i,m-1}}(r^{i,m-1},\bar{P}) \nonumber
	    \\ & =\frac{81}{cK^3} \sum_{i=1}^{K^3}\sum_{h,s,a} W^{\tilde{\pi}^{i,m-1}}(\textbf{1}_{h,s,a},\bar{P})\cdot \min \left\{  \frac{1}{\sum_{j=1}^{i-1}W^{\tilde{\pi}^{j,m-1}}(\textbf{1}_{h,s,a},p^{m-1})  }     ,1\right\}\nonumber
	    \\ & = \frac{81}{cK^3}\sum_{h,s,a}\sum_{i=1}^{K^3}\sum_{h,s,a} W^{\tilde{\pi}^{i,m-1}}(\textbf{1}_{h,s,a},\bar{P})\cdot \min \left\{  \frac{1}{\sum_{j=1}^{i-1}W^{\tilde{\pi}^{j,m-1}}(\textbf{1}_{h,s,a},p^{m-1})  }     ,1\right\}\nonumber
	    \\ & \leq \frac{243}{cK^3}\sum_{h,s,a}\sum_{i=1}^{K^3}\sum_{h,s,a} W^{\tilde{\pi}^{i,m-1}}(\textbf{1}_{h,s,a},\bar{P})\cdot \min \left\{  \frac{1}{\sum_{j=1}^{i-1}W^{\tilde{\pi}^{j,m-1}}(\textbf{1}_{h,s,a},\bar{P})  }     ,1\right\}\nonumber
	    \\ & \leq \frac{243SAH\ln(K)}{cK^3}\label{eq:bbbr}
	\end{align}
	where the second line is by the \emph{tightness} (w.r.t. $\bar{P}$) of $\mathcal{P}^{m-1}$, and the last line is by the fact that for any  non-negative $\{x_i\}_{i=1}^n$
	\begin{align}
	    \sum_{i =1}^nx_i \cdot \min\left\{\frac{1}{\sum_{j=1}^{i-1}x_j},1\right\} & \leq  2+ 2\sum_{i= 1}^n\left(\ln\left(\sum_{j=1}^{i}x_i\right)-\ln\left(\sum_{j=1}^{i-1}x_j\right)\right)\mathbb{I}\left[\left(\sum_{j=1}^{i-1}x_j\right)\geq 1\right]\nonumber
	    \\ & \leq 2+ 2\ln\left(\sum_{i=1}^n x_i\right).\nonumber
	\end{align}
	
	By definition of $r^{K^3,m-1}$, we have that for any $(h,s,a)$
	\begin{align}
	    r^{K^3+1,m-1}_h(s,a) & = \min\left\{ \frac{1}{\sum_{j=1}^{K^3}W^{\tilde{\pi}^{j,m-1}}(\textbf{1}_{h,s,a} ,p^{m-1}) } ,1\right\}\nonumber	    \\ & \geq \frac{1}{3} \min\left\{ \frac{1}{\sum_{j=1}^{K^3}W^{\tilde{\pi}^{j,m-1}}(\textbf{1}_{h,s,a} ,\bar{P}) } ,1\right\} = \frac{1}{3}\min\left\{ \frac{1}{K^3W^{\pi^m}(\textbf{1}_{h,s,a} ,\bar{P}) } ,1\right\}.\label{eq:5251}
	\end{align}
	By \eqref{eq:bbbr} and \eqref{eq:5251}, for any $\pi \in \Pi(\mathcal{P}^{m-1})$ it holds that 
	\begin{align}
	   \sum_{h,s,a}W^{\pi}(\textbf{1}_{h,s,a},\bar{P})\cdot  \min\left\{ \frac{1}{K^3W^{\pi^m}(\textbf{1}_{h,s,a} ,\bar{P}) } ,1\right\} & \leq  3 \sum_{h,s,a}W^{\pi}(\textbf{1}_{h,s,a},\bar{P})r^{K^3+1,m-1}_h(s,a)  \leq \frac{729SAH\ln(K)}{cK^3}.\label{eq:samplerange}
	\end{align}

Note that $\pi^m$ is executed for $K_{m-2H}$ rounds. By Lemma~\ref{lemma:ete}, with probability $1-4SAH\delta$, it holds that 
\begin{align}
    \check{N}_h^m(s,a)  \geq \frac{1}{3}K_{m-2H}W^{\pi^m}(1_{h,s,a},\bar{P})-\iota.\label{eq:countlb}
\end{align}
	
Fix $\pi \in \Pi(r,\mathcal{P}^{m-1})$.
	Let $\{f_h(\cdot)\}_{h=1}^{S}$ be the value function under $\pi$ and $\bar{P}$.
 For any $P'\in \mathcal{P}^m$, by policy difference lemma, we have that
\begin{align}
 & \left|W^{\pi}(r,P') - W^{\pi}(r,\bar{P})\right| \nonumber
    \\ & =\left|\sum_{h,s,a}W^{\pi}(\textbf{1}_{h,s,a},P')\cdot (P'_{h,s,a}-\bar{P}_{h,s,a})f_{h+1} \right|\nonumber
    \\ & \leq \underbrace{\left|\sum_{h,s,a}W^{\pi}(\textbf{1}_{h,s,a},P') (P'_{h,s,a}-\bar{P}_{h,s,a})v^{m-1}_{h+1}\right|}_{\mathbf{Term.1}} + \underbrace{\left|\sum_{h,s,a}W^{\pi}(\textbf{1}_{h,s,a},P')(P'_{h,s,a}-\bar{P}_{h,s,a})(f_{h+1}-v^{m-1}_{h+1})\right|}_{\mathbf{Term.2}}.\label{eq:decom}
\end{align}

By the definition of $\mathcal{P}^{m}$ and $\mathcal{G}$, we have that
\begin{align}
    \mathbf{Term.1}  & = \left|\sum_{h,s,a}W^{\pi}(\textbf{1}_{h,s,a},P') (P'_{h,s,a}-\check{P}^{m}_{h,s,a} +\check{P}_{h,s,a}^m-P_{h,s,a})v^{m-1}_{h+1}\right|  \nonumber
    \\ & \leq \sum_{h,s,a}W^{\pi}(\textbf{1}_{h,s,a},P') \cdot \left(  5\sqrt{\frac{\mathbb{V}(\check{P}^m_{h,s,a},v^{m-1}_{h+1})\iota}{\check{N}^m_h(s,a) }}+5\sqrt{\frac{\mathbb{V}(\bar{P}_{h,s,a},v^{m-1}_{h+1})\iota}{\check{N}^m_h(s,a) }} +\frac{8\iota}{\check{N}^m_{h}(s,a)}   \right) \nonumber
    \\ & \leq O\left(\sqrt{\sum_{h,s,a}\frac{W^{\pi}(\textbf{1}_{h,s,a},\bar{P})\iota}{\check{N}_h^m(s,a)} } \cdot \sqrt{ \sum_{h,s,a}W^{\pi}(\textbf{1}_{h,s,a},\bar{P})\cdot \left(\mathbb{V}(\check{P}^m_{h,s,a},v^{m-1}_{h+1})+\mathbb{V}(\bar{P}_{h,s,a},v^{m-1}_{h+1})\right) }\right) \nonumber
    \\ & \quad \quad \quad \quad \quad \quad \quad \quad \quad \quad \quad \quad \quad \quad \quad \quad \quad \quad \quad \quad \quad \quad \quad \quad \quad \quad \quad \quad  +O\left(\sum_{h,s,a}\frac{W^{\pi}(\textbf{1}_{h,s,a},\bar{P})\iota}{\check{N}^m_h(s,a)}\right)
\end{align}

Define $T_1 = \sum_{h,s,a}\frac{W^{\pi}(\textbf{1}_{h,s,a},\bar{P})\iota}{\check{N}^m_h(s,a)}$,  $T_2 = \sum_{h,s,a}W^{\pi}(\textbf{1}_{h,s,a},\bar{P})\cdot \mathbb{V}(\bar{P}_{h,s,a},v^{m-1}_{h+1})$ and $T_2 =\sum_{h,s,a}W^{\pi}(\textbf{1}_{h,s,a},\bar{P})\cdot \mathbb{V}(\check{P}^m_{h,s,a},v^{m-1}_{h+1})$.
	
\paragraph{Bound of $T_1$}	
By \eqref{eq:samplerange} and \eqref{eq:countlb}, we have that
\begin{align}
    T_1 & \leq 3\sum_{h,s,a} \frac{W^{\pi}(\textbf{1}_{h,s,a},\bar{P})}{\max\{ K_{m-2H}W^{\pi^m}(\textbf{1}_{h,s,a},\bar{P})-3\iota ,1\}} \nonumber
    \\ & = \frac{3K^3}{K_{m-2H}}\sum_{h,s,a}W^{\pi}(\textbf{1}_{h,s,a},\bar{P})\cdot \min\left\{  \frac{1}{K^3W^{\pi^m}(\textbf{1}_{h,s,a},\bar{P})-3K^3\iota/K_{m-2H}} ,\frac{K_{m-2H}}{K^3}\right\}\nonumber
    \\ & \leq \frac{3K^3}{K_{m-2H}}\sum_{h,s,a}W^{\pi}(\textbf{1}_{h,s,a},\bar{P})\cdot\left(\min\left\{ \frac{2}{K^3W^{\pi^m}(\textbf{1}_{h,s,a},\bar{P})},1\right\}\cdot \mathbb{I}\left[K_{m-2H}W^{\pi^m}(\textbf{1}_{h,s,a},\bar{P})\geq 6\iota\right] \right)\nonumber
    \\ &  \quad \quad \quad \quad\quad \quad \quad \quad\quad \quad \quad \quad\quad \quad \quad \quad  \quad \quad \quad + 3\sum_{h,s,a}W^{\pi}(\textbf{1}_{h,s,a},\bar{P})\mathbb{I}\left[ K_{m-2H}W^{\pi^m}(\textbf{1}_{h,s,a},\bar{P})< 6\iota\ \right]\nonumber
    \\ & \leq \frac{3K^3}{K_{m-2H}}\cdot \frac{729SAH\ln(K)}{cK^3} + \frac{18SAH\iota}{K_{m-2H}}\nonumber
    \\ & = O\left( \frac{SAH\ln(K)\iota }{K_{m-2H}}\right).\label{eq:bdt1}
\end{align}
	
\paragraph{Bound of $T_2$}
\begin{align}
    T_2 &  = \sum_{h,s,a}W^{\pi}(\textbf{1}_{h,s,a},\bar{P})\cdot \mathbb{V}(\bar{P}_{h,s,a},v^{m-1}_{h+1}) \nonumber
    \\ &  = \sum_{h,s,a}W^{\pi}(\textbf{1}_{h,s,a},\bar{P})\cdot \left( \bar{P}_{h,s,a}(v_{h+1}^{m-1})^2 - (\bar{P}_{h,s,a}v_{h+1}^{m-1})^2    \right) \nonumber
    \\ & \leq \sum_{h,s,a}W^{\pi}(\textbf{1}_{h,s,a},\bar{P})\cdot \left( (v^{m-1}_{h}(s))^2 - (\bar{P}_{h,s,a}v_{h+1}^{m-1})^2    \right) + H^2 \nonumber
    \\ & \leq H \sum_{h=1}^H \mathbb{E}_{\pi,\bar{P}}\left[ |v_{h}^{m-1}(s_h)- \bar{P}_{h,s_h,a_h}v_{h+1}^{m-1}|\right]+H^2 \nonumber
    \\ & = H \sum_{h=1}^H \mathbb{E}_{\pi,\bar{P}}\left[ v_{h}^{m-1}(s_h)- \bar{P}_{h,s_h,a_h}v_{h+1}^{m-1}\right]+H^2   \label{eq:exp351}
    \\ & \leq H\sum_{h=1}^H  \mathbb{E}_{\pi,\bar{P}}[r_{h}(s_h,a_h)]+2H^2\nonumber
    \\ & \leq 4H^2.\label{eq:boundt2}
\end{align}
Here \eqref{eq:exp351} is by the fact that $v_{h}^{m-1}$ is the optimal value function with respect to $\mathcal{P}^{m-1}$ and $\bar{P}\in \mathcal{P}^{m-1}$.

\paragraph{Bound of $T_3$}

By Lemma~\ref{lemma:bennet}, with probability $1-4S^2AH\delta$, it holds that
\begin{align}
    \left|\check{P}^m_{h,s,a,s'}-\bar{P}_{h,s,a,s'}\right| \leq 4\sqrt{\frac{\bar{P}_{h,s,a,s'}\iota}{\check{N}^{m}_{h}(s,a)}} +\frac{3\iota}{\check{N}_h^m(s,a)}\leq 2\bar{P}_{h,s,a,s'} + \frac{5\iota}{\check{N}_h^m(s,a)}.\nonumber
\end{align}
As a result, we have that 
\begin{align}
    T_3& =  \sum_{h,s,a}W^{\pi}(\textbf{1}_{h,s,a},\bar{P})\cdot \mathbb{V}(\check{P}^m_{h,s,a},v^{m-1}_{h+1})\nonumber
    \\ & \leq  \sum_{h,s,a}W^{\pi}(\textbf{1}_{h,s,a},\bar{P})\cdot \sum_{s'}\check{P}^{m}_{h,s,a,s'}\left( v_{h+1}^{m-1}(s') -  \bar{P}_{h,s,a,s'}v_{h+1}^{m-1} \right)^2 \nonumber
\\ & \leq \sum_{h,s,a}W^{\pi}(\textbf{1}_{h,s,a},\bar{P})\cdot \sum_{s'}\bar{P}_{h,s,a,s'}^{m}\left( v_{h+1}^{m-1}(s') -  \bar{P}_{h,s,a,s'}v_{h+1}^{m-1} \right)^2  + \sum_{h,s,a}W^{\pi}(\textbf{1}_{h,s,a},\bar{P})\cdot \frac{5H^2\iota}{\check{N}_h^m(s,a)}\nonumber
\\ & =4T_2+ 5H^2\iota T_1\nonumber
\\ & \leq  O\left(H^2 + \frac{SAH^2\ln(K)\iota^2}{K_{2m-H}}\right).\label{eq:boundt3}
\end{align}	
	
By \eqref{eq:bdt1}, \eqref{eq:boundt2} and \eqref{eq:boundt3}, $\mathbf{Term.1}$ is bounded by
\begin{align}
    \mathbf{Term.1}\leq O\left( \sqrt{\frac{SAH^3\ln(K)\iota^2}{K_{m-2H}}}+  \frac{SAH^2\ln(K)\iota}{K_{m-2H}} \right).\label{eq:bdterm1}
\end{align}

To bound $\mathbf{Term.2}$, by definition of $\mathcal{P}^m$ and $\mathcal{G}$, we have
\begin{align}
    \mathbf{Term.2} & = \left|\sum_{h,s,a}W^{\pi}(\textbf{1}_{h,s,a},P')(P'_{h,s,a}-\bar{P}_{h,s,a})(f_{h+1}-v^{m-1}_{h+1})\right| \nonumber
  \\ & \leq \sum_{h,s,a}W^{\pi}(\textbf{1}_{h,s,a},P')\sum_{s'}\left(10\sqrt{\frac{\bar{P}_{h,s,a,s'}\iota}{N_h^m(s,a)}}+\frac{6\iota}{N_h^m(s,a)}\right)\cdot |f_{h+1}(s')-v_{h+1}^{m-1}(s')-l| \nonumber
  \\ & \leq O\left(    \sum_{h,s,a}W^{\pi}(\textbf{1}_{h,s,a},\bar{P})\sum_{s'}\sqrt{\frac{\bar{P}_{h,s,a,s'}\iota}{N_h^m(s,a)}} |f_{h+1}(s')-v_{h+1}^{m-1}(s')-l_h(s,a)| \right) \nonumber
  \\ &\quad\quad \quad \quad \quad\quad \quad \quad\quad\quad \quad \quad\quad\quad \quad \quad\quad\quad  + O\left(    \sum_{h,s,a}W^{\pi}(\textbf{1}_{h,s,a},\bar{P})\frac{SH\iota}{N_h^m(s,a)} \right) ,\label{eq:ddcomt2}
\end{align}
where $l_h(s,a) = \bar{P}_{h,s,a}(f_{h+1}-v_{h+1}^m)$. By \eqref{eq:bdt1}, the second term in \eqref{eq:ddcomt2} is  bounded by $O\left( \frac{SAH\ln(K)\iota}{K_{m-2H}}\right)$. To bound the   the first term in \eqref{eq:ddcomt2}, by Cauchy's inequality, we have that
\begin{align}
 &    O\left(      \sum_{h,s,a}W^{\pi}(\textbf{1}_{h,s,a},\bar{P})\sqrt{\frac{S\mathbb{V}(\bar{P}_{h,s,a}, f_{h+1}-v_{h+1}^{m-1})\iota}{N_h^m(s,a)}} \right)\nonumber
 \\ & \leq O\left(\sqrt{\frac{SW^{\pi}(\textbf{1}_{h,s,a},\bar{P})\iota}{N_h^m(s,a)}}\cdot \sqrt{\sum_{h,s,a}W^{\pi}(\textbf{1}_{h,s,a},\bar{P})\mathbb{V}(\bar{P}_{h,s,a}, f_{h+1}-v_{h+1}^{m-1})} \right)\nonumber
 \\ & \leq O\left(\sqrt{\frac{S^2AH\ln(K)\iota^2}{K_{m-2H}}}\cdot \sqrt{\sum_{h,s,a}W^{\pi}(\textbf{1}_{h,s,a},\bar{P})\mathbb{V}(\bar{P}_{h,s,a}, f_{h+1}-v_{h+1}^{m-1})} \right),\nonumber
\end{align}
where the last line is by \eqref{eq:bdt1}. Continuing the computation:
\begin{align}
& \sum_{h,s,a}W^{\pi}(\textbf{1}_{h,s,a},\bar{P})\mathbb{V}(\bar{P}_{h,s,a}, f_{h+1}-v_{h+1}^{m-1}) \nonumber
\\ & = \sum_{h,s,a}W^{\pi}(\textbf{1}_{h,s,a},\bar{P})\left( \bar{P}_{h,s,a}(f_{h+1}-v_{h+1^{m-1}})^2  -(\bar{P}_{h,s,a}f_{h+1}-\bar{P}_{h,s,a}v_{h+1}^{m-1})^2\right) \nonumber
\\ & \leq \mathbb{E}_{\pi,\bar{P}}\left[ \sum_{h=1}^{H}\left((f_{h+1}(s_{h+1})- v_{h+1}^{m-1}(s_{h+1} )^2 - (\sum_{a}\pi_h(a|s)\bar{P}_{h,s,a}(f_{h+1}-v_{h+1}) )^2 \right)       \right]\label{eq:vatotal}
\\ & \leq (v_{1}^{m-1}(s_1)-f_1(s_1))^2+ H\mathbb{E}_{\pi,\bar{P}}\left[  \sum_{h=1}^{H} \left|f_{h}(s_{h})- v_{h}^{m-1}(s_{h})  - \sum_{a}\pi_h(a|s_h)\bar{P}_{h,s_h,a}(f_{h+1}-v^{m-1}_{h+1})  \right|          \right] \nonumber
\\ & \leq (v_{1}^{m-1}(s_1)-f_1(s_1))^2+ H\mathbb{E}_{\pi,\bar{P}}\left[  \sum_{h=1}^{H} v^{m-1}_{h}(s_{h})-  \sum_{a}\pi_h(a|s_h)\left(r_h(s_h,a)+\bar{P}_{h,s_h,a}v^{m-1}_{h+1})\right)       \right] \label{eq:exp454}
\\ & \leq  (v_{1}^{m-1}(s_1)-f_1(s_1))^2+H(v_1^{m-1}(s_1) - f_1(s_1))\nonumber
\\ & \leq   2 H(v_1^{m-1}(s_1) - f_1(s_1)).\nonumber
\end{align}
Here \eqref{eq:vatotal} holds by the fact that $\mathrm{Var}(X)\geq \mathbb{E}_{Y}[\mathrm{Var}(X|Y)]$ for any random variables $X$ and $Y$ (recalling that $\mathrm{Var}(X)$ denotes the variance of $X$), and \eqref{eq:exp454} is by the fact that $v^{m-1}_h(s_h)\geq \sum_{a}\pi_h(a|s_h) (r_h(s_h,a)+\bar{P}_{h,s,a}v_{h+1}^{m-1})$ and $f_h(s_h)= \sum_{a}\pi_h(a|s_h) (r_h(s_h,a)+\bar{P}_{h,s,a}f_{h+1})$ for any $1\leq h \leq H$.

Because $\pi\in \Pi(r,\mathcal{P}^{m-1})$, we learn that $v_1^{m-1}-f_1(s_1)\leq \mathrm{gap}^{m}$. By Lemma~\ref{lemma:raw++}, we have that for any $m\geq H+1$, $N_{h}^{m-1}(s,a)\geq \frac{ck_1}{27SA}\max_{\pi}W^{\pi}(\textbf{1}_{h,s,a},\bar{P})-4\iota - \frac{36C_1SAH^3\iota}{27}$. With similar analysis, and noting that $\|p'_{h,s,a}-p''_{h,s,a}\|_{1}\leq O(\sqrt{S\iota/N_h^{m-1}(s,a)})$ for any $p',p''\in \mathcal{P}^{m-1}$, we have 
\begin{align}
    \mathrm{gap}^{m}&\leq O\left(\max_{\pi'}\sum_{h,s,a}W^{\pi'}(\textbf{1}_{h,s,a},\bar{P}) \sqrt{\frac{SH^2\iota}{N_h^{m-1}(s,a)}}\right)  \nonumber
    \\ & \leq O\left(\sqrt{\frac{S^4A^3H^4\iota}{k_1}}+\frac{S^{\frac{7}{2}}A^{3}H^5\iota^{\frac{3}{2}}}{k_1} \right).
\end{align}

As a result, we have that 
\begin{align}
\mathbf{Term.2} \leq O\left( \sqrt{\frac{S^\frac{11}{2}A^4H^7\ln(K)\iota^{\frac{5}{2}}}{ K_{m-2H}k_1 }} +\sqrt{\frac{S^4A^{\frac{5}{2}}H^4\ln(K)\iota^\frac{3}{2}}{ K_{m-2H}\sqrt{k_1} }}  + \frac{S^2AH\ln(K)\iota}{K_{m-2H}} \right) .\label{eq:bdterm2}
\end{align}

Putting all together, for any $\pi\in \Pi(r,\mathcal{P}^{m-1})$ and any $P'\in\mathcal{P}^m$, we have
\begin{align}
  & | W^{\pi}(r,P') - W^{\pi}(r,\bar{P})|\nonumber
  \\ & \leq O\left(\sqrt{\frac{SAH^3\ln(K)\iota^2}{K_{m-2H}}}+  \frac{SAH^2\ln(K)\iota}{K_{m-2H}}+\sqrt{\frac{S^\frac{11}{2}A^4H^7\ln(K)\iota^{\frac{5}{2}}}{ K_{m-2H}k_1 }} +\sqrt{\frac{S^4A^{\frac{5}{2}}H^4\ln(K)\iota^\frac{3}{2}}{ K_{m-2H}\sqrt{k_1} }} \right).\nonumber
\end{align}

By definition, there exists $P',P''$ such that $U^{\pi}(\mathcal{P}^{m}) = W^{\pi}(r,P')$ and $L^{\pi}(\mathcal{P}^{m}) = W^{\pi}(r,P'')$. Therefore,  
\begin{align}
  & | U^{\pi}(\mathcal{P}^{m}) -L^{\pi}(\mathcal{P}^{m})|\nonumber\\ & \leq  O\left(\sqrt{\frac{SAH^3\ln(K)\iota^2}{K_{m-2H}}}+  \frac{SAH^2\ln(K)\iota}{K_{m-2H}}+\sqrt{\frac{S^\frac{11}{2}A^4H^7\ln(K)\iota^{\frac{5}{2}}}{ K_{m-2H}k_1 }} +\sqrt{\frac{S^4A^{\frac{5}{2}}H^4\ln(K)\iota^\frac{3}{2}}{ K_{m-2H}\sqrt{k_1} }} \right).\nonumber
\end{align}
Taking maximization over $\pi\in \Pi(r,\mathcal{P}^{m-1})$ we finish the proof.
\end{proof}

\subsubsection{Statement and Proof of Lemma~\ref{lemma:raw++}}

\begin{lemma}\label{lemma:raw++}
Given a dataset $\mathcal{D}$ and $k\geq 0$, let $\mathcal{D}'$ be the output by running Algorithm~\ref{alg:raw_exploration} with input $(r,\mathcal{D},k)$. 
Let $\{N_h(s,a,s')\}(\{N'_{h}(s,a,s')\})$ be the counts with respect to $\mathcal{D}(\mathcal{D}')$. Let $\mathcal{W} = \{(h,s,a,s')| N_h(s,a,s')\geq C_1H^2\iota\}$ and
$\mathcal{W}' = \{(h,s,a,s') | N'_{h}(s,a,s')\geq C_1H^2\iota \}$. Let $\bar{p}=\mathtt{clip}(P,\mathcal{W})$.
With probability $1-4S^2AH^2\delta$, it holds that 
\begin{align}
 &  \max_{\pi \in \Pi^*}\mathrm{Pr}_{\pi}\left[ \exists h'\in [h] , (h',s_{h'},a_{h'},s_{h'+1})\notin \mathcal{W}' \right] \leq \frac{36C_1S^2A^2H^3\iota  }{k} ,\label{eq:poo2}
\end{align}
where $\Pi^*$ is the set of optimal policies. Moreover, if $\mathcal{D}=\emptyset$ and $u=0$, with probability $1-4S^2AH^2\delta$ it holds that 
\begin{align}
N'_{h,s,a} \geq \frac{ck}{27SA}\max_{\pi}W^{\pi}(\textbf{1}_{h,s,a},P)-4\iota-\frac{36C_1SAH^3\iota}{27}\label{eq:poo1}
\end{align}
for any $1\leq h \leq H$ .

\end{lemma}
\begin{proof}


For $h'=1,2,...,H$, we denote $\mathcal{D}^{h'}$ as the value of $\mathcal{D}$ after the $h'$-th batch in Algorithm~\ref{alg:raw_exploration}. Similarly, we define $\{N^{h'}_{h}(s,a,s') \}$, $\{N^{h'}_h(s,a)\}$ and $\{\hat{p}^{h'}_{h,s,a}\}$ be respectively the value of $\{N_h(s,a,s')\}$ , $\{N_h(s,a)\}$ and $\{\hat{p}_{h,s,a}\}$ after the $h'$-th batch. Note that $\mathcal{P}^{h'}=\mathtt{CR}(\mathcal{D}^{h'})$ is the value of $\mathcal{P}$ after the $h'$-th batch. 

Define $\mathcal{W}^{h'}: = \{(h,s,a,s'): N^{h'}_h(s,a,s')\geq C_1 H^2\iota \}$ and $P^{h'}=\mathtt{clip}(P,\mathcal{W}^{h'})$.
Let $p^{h'}\in \mathcal{P}^{h'-1}$ be the transition model chosen at line 8 Algorithm~\ref{alg:raw_exploration}.

Using Lemma~\ref{lemma:bennet} and Lemma~\ref{lemma:ete}, with probability $1-4S^2AH^2\delta$, for any $(h,s,a,s')\in \mathcal{W}^{h'}$, it holds that
\begin{align}
    \left|P^{h'}_{h,s,a,s'}- \hat{p}_{h,s,a,s'}^{h'}  \right| &\leq
    \sqrt{\frac{4P^{h'}_{h,s,a,s'}\iota}{N^{h'}_h(s,a)}}+\frac{\iota}{3N^{h'}_h(s,a)}\nonumber 
	    \\ & \leq \frac{1}{3H}P^{h'}_{h,s,a,s'},\nonumber
\end{align}
where in the last inequality, we use Lemma~\ref{lemma:ete} to get that $N_h^{h'}P_{h,s,a,s'}^{h'}\geq \frac{1}{3}N_h^{h'}(s,a,s')-\iota \geq 64H^2\iota$ with probability $1-\delta$.

It then holds that $P^{h'}\in \mathcal{P}^{h'}$ for each $h'$. Moreover, noting that for any $p\in \mathcal{P}^{h'}$ and $(h,s,a,s')\in \mathcal{W}^{h'}$, with similar computation it holds that
\begin{align}
    \left|p_{h,s,a,s'}-P^{h'}_{h,s,a,s'} \right|& \leq  \left|p_{h,s,a,s'}-\hat{p}^{h'}_{h,s,a,s'} \right| +  \left|\hat{p}_{h,s,a,s'}-P^{h'}_{h,s,a,s'} \right| \nonumber
    \\ & \leq \frac{1}{3H}P^{h'}_{h,s,a,s'}+\frac{1}{3H}\hat{p}^{h'}_{h,s,a,s'}\nonumber
    \\ & \leq \left(\frac{2}{3H}+\frac{1}{9H^2}\right) P^{h'}_{h,s,a,s'}\nonumber 
\end{align}
As a result, $\mathcal{P}^{h'}$ is \emph{tight} with respect to $P^{h'}$.

Fix $h\in [H]$.
Recall that 
\begin{align}
 &     \pi^{h,s,a}= \mathtt{Policy\,Search}(\textbf{1}_{h,s,a},\mathcal{P}^{h-1});\nonumber
 \\ & \{\tilde{\pi}^h,p^{h}\} = \mathtt{Sum}\left( \left\{\frac{1}{SA},\pi_{h,s,a},p^{h} \right\}_{h,s,a}      \right).
\end{align}
Recall that, for the first $h-1$ steps  $\pi^h$ is the policy which is the same as $\tilde{\pi}^h$, and for the left $H-h+1$ steps, $\pi^h$ is the uniformly random policy. 

We first show that the $h$-th layer is well explored.
By the property of $\mathtt{Policy\,Search}$ and $\mathtt{Sum}$ (see Lemma~\ref{lemma:cei11} and \ref{lemma:stp}), there exists a constant $c>0$ such that\footnote{We omit $\epsilon$ for convenience. By setting $\epsilon = 1/(SAHK)^{10}$, it is easy to verify the error only leads to a lower order term.}
\begin{align}
   & W^{\pi^{h,s,a}}(\textbf{1}_{h,s,a},P^{h-1})\geq c\max_{\pi\in \Pi(r,\mathcal{P}^{h-1})}W^{\pi}(\textbf{1}_{h,s,a},P^{h-1}) \nonumber
   \\ & W^{\pi^h}(\textbf{1}_{h,s,a},p^{h}) = \frac{1}{SA}W^{\pi_{h,s,a}}(\textbf{1}_{h,s,a},p^{h}),\forall (s,a)\in \mathcal{S}\times \mathcal{A}.\nonumber
\end{align}

Noting that $p^{h}\in \mathcal{P}^{h-1}$ and $\mathcal{P}^{h-1}$ is \emph{tight} with respect to $P^{h-1}$, by Lemma~\ref{lemma:tight} we obtain that 
\begin{align}
    W^{\pi^h}(\textbf{1}_{h,s,a},P^{h-1})\geq \frac{c}{9SA}\max_{\pi\in \Pi(r,\mathcal{P}^{h-1})}W^{\pi}(\textbf{1}_{h,s,a},P^{h-1})\label{eq:xxx1}
\end{align}
for any $(s,a)\in \mathcal{S}\times \mathcal{A}$.

Using Lemma~\ref{lemma:ete}, with probability $1-4SA\delta$, the count of $(h,s,a)$ in the $h$-th batch is at least $\frac{ck}{27SA}\cdot \max_{\pi\in \Pi(r,\mathcal{P}^{h-1})}W^{\pi}(\textbf{1}_{h,s,a},P^{h-1}) -\iota$. As a result, we have that
\begin{align}
    N^{h}_{h}(s,a)\geq \frac{ck}{27SA}\cdot \max_{\pi\in \Pi(r,\mathcal{P}^{h-1})}W^{\pi}(\textbf{1}_{h,s,a},P^{h-1}) -4\iota \label{eq:lbnn}
\end{align}
for any $(s,a)\in\mathcal{S}\times \mathcal{A}$.

In the meantime,   if $(h,s,a,s')\notin \mathcal{W}^{h}$ , we have that $N^{h}_{h}(s,a,s')\leq C_1H^2\iota$. Using Lemma~\ref{lemma:ete}, with probability $1-\delta$, we have that
\begin{align}
    kW^{\pi^h}(\textbf{1}_{h,s,a},P^{h-1})P_{h,s,a,s'}\leq 3C_1H^2\iota + 4\iota.\label{eq:xxx2}
\end{align}
Combining \eqref{eq:xxx1} and \eqref{eq:xxx2}, we have that
\begin{align}
    \max_{\pi\in \Pi(r,\mathcal{P}^{h-1})}\mathbb{P}_{\pi,P^{h-1}}\left[ (s_{h},a_{h},s_{h+1})=(s,a,s') \right]  & =\max_{\pi\in \Pi(r,\mathcal{P}^{h-1})}W^{\pi}(\textbf{1}_{h,s,a},P^{h-1}) P_{h,s,a,s'} \nonumber
    \\ & \leq \frac{9SA}{c}W^{\pi^h}(\textbf{1}_{h'',s,a},P^{h-1})P_{h'',s,a,s'}\nonumber
    \\ & \leq \frac{9SA}{c}\cdot \frac{3C_1H^2\iota + 4\iota}{k}\nonumber
    \\ & \leq \frac{36C_1SAH^2\iota }{k}.\label{eq:boundxx1}
\end{align}
With an union bound over all $(h,s,a,s')\notin \mathcal{W}^h$, we have that
\begin{align}
     \max_{\pi\in \Pi(r,\mathcal{P}^{h-1})}\mathbb{P}_{\pi,P^{h-1}}\left[ (h,s_h,a_h,s_{h+1})\notin \mathcal{W}^h \right]  \leq \frac{36C_1S^2A^2H^2\iota}{k}.\label{eq:boundxx2}
\end{align}

Note that $\mathcal{W}^{h}$ is non-decreasing in $h$. For any  $\pi\in \cap_{h=1}^{h'} \Pi(r,\mathcal{P}^{h})$, it holds that
\begin{align}
&\mathbb{P}_{\pi,P}\left[\exists h'\leq H, (h',s_{h'},a_{h'},s_{h'+1})\notin \mathcal{W}^{H}  \right] \nonumber
\\&=\mathbb{P}_{\pi,P^{H}}\left[\exists h'\leq H, (h',s_{h'},a_{h'},s_{h'+1})\notin \mathcal{W}^{H}  \right] \nonumber
\\ & = \sum_{h'=1}^{H}\mathbb{P}_{\pi,P^{H}}\left[(h',s_{h'},a_{h'},s_{h'+1})\in \mathcal{W}^{H},\forall 1\leq h'<h, (h,s_{h},a_{h},s_{h+1})\notin \mathcal{W}^{H}  \right] \nonumber
\\ & \leq \sum_{h=1}^{H}
\max_{\pi\in \Pi(r,\mathcal{P}^{h-1})}\mathbb{P}_{\pi,P^{h-1}}\left[ (h,s_{h},a_{h},s_{h+1})\notin \mathcal{W}^{h} \right]  \nonumber
\\ & \leq \frac{36C_1S^2A^2H^3\iota }{k}.\label{eq:pfpoo1}
\end{align}

Recall that $\Pi(r,\mathcal{P}): =\{ \pi | U^{\pi}(r+\textbf{1}_{z},\mathcal{P})\geq \max_{\pi}L^{\pi}(r,\mathcal{P})         \}$. 
Because $P^h\in \mathcal{P}^h$ for any $h$, for any optimal policy $\pi^*$ and any policy $\pi'$, we have that  $U^{\pi^*}(r+\textbf{1}_{z},\mathcal{P})\geq V^*_1(s_1)\geq W^{\pi'}(r,P^h) \geq  L^{\pi'}(r,\mathcal{P})$. Therefore, $\pi^*\in \Pi(\mathcal{P}^{h})$ for any $1\leq h \leq H$. By \eqref{eq:pfpoo1}, \eqref{eq:poo2} is proven.

In the case $u=0$, we have that $\Pi(u,\mathcal{P}^{h})=\overline{\Pi}$ for $1\leq h \leq H$,
where $\overline{\Pi}$ is the set of all possible policies. By \eqref{eq:lbnn}, we have that
\begin{align}
   & N_{h}^{h}(s,a) 
    \\ & \geq \frac{ck}{27SA}\max_{\pi}W^{\pi}(\textbf{1}_{h,s,a},P^{h})-4\iota 
    \\ & \geq \frac{ck}{27SA}\max_{\pi}W^{\pi}(\textbf{1}_{h,s,a},P^)-4\iota - \frac{ck}{27SA}\max_{\pi}\mathbb{P}_{\pi,P}\left[\exists h'\in [H], (h',s_{h'},a_{h'},s_{h'+1})\notin \mathcal{W}^{H} \right]\nonumber
    \\ & \geq \frac{ck}{27SA}\max_{\pi}W^{\pi}(\textbf{1}_{h,s,a},P)-4\iota-\frac{36C_1SAH^3\iota}{27}.\nonumber
\end{align}
The proof is completed by noting that $N'_h(s,a)\geq N_h^h(s,a)$.

\end{proof}

\subsection{Statement and Proof of Lemma~\ref{lemma:tight}}

\begin{lemma}\label{lemma:tight}
Suppose $\mathcal{P}$ is \emph{tight} with respect to $p$. Then we have that
\begin{align}
3 W^{\pi}(\textbf{1}_{h,s,a},p)\geq     W^{\pi}(\textbf{1}_{h,s,a},p')\geq \frac{1}{3}W^{\pi}(\textbf{1}_{h,s,a},p) \label{eq:llrr}
\end{align}
for any $p'\in \mathcal{P}$, policy $\pi$ and $(h,s,a)$.
\end{lemma}
\begin{proof}
For each trajectory $L=(s_1,a_1,...,s_H,a_H,s_{H+1})$ such that $s_h \neq z$ for $1\leq h \leq H+1$, we have that
\begin{align}
    \mathbb{P}_{\pi,p}[L] = \pi_{h=1}^{H}\pi_h(a_h|s_h)p_{h,s_h,a_h,s_{h+1}}\geq e^{-\frac{H}{H}}\mathbb{P}_{\pi,p}[L] = \pi_{h=1}^{H}\pi_h(a_h|s_h)p'_{h,s_h,a_h,s_{h+1}} \geq \frac{1}{3} \mathbb{P}_{\pi,p'}[L].\nonumber
\end{align}
So the left side of \eqref{eq:llrr} is proven. By reversing $p$ and $p'$ the right side follows.

\end{proof}

\section{ Other Missing Proofs}

	\subsection{Proof of Lemma~\ref{lemma:design}}\label{app:pflemma1}
	\textbf{Lemma~\ref{lemma:design} (restated)}\emph{
	Let $d>0$ be an integer. 
Let $\mathcal{X}\subset (\Delta^d)^{m}$. Then there exists a distribution $\mathcal{D}$ over $\mathcal{X}$, such that 
\begin{align}
\max_{x=\{x_i\}_{i=1}^{dm}\in \mathcal{X}}\sum_{i=1}^{dm}\frac{x_i}{y_i} =  md,\nonumber
\end{align}
where $y =\{y_i\}_{i=1}^{dm}= \mathbb{E}_{x\sim \mathcal{D}}[x]$. Moreover, if $\mathcal{X}$ has a boundary set $\partial \mathcal{X}$ with finite cardinality, we can find $\mathcal{D}$ in $\mathrm{poly}(|\partial \mathcal{X}|)$ time.
	}
\begin{proof}
Note that $\mathcal{X}$ is always bounded. Without loss of generality, we assume $\mathcal{X}$ is a discrete set
 with $\mathcal{X}=\{x^1,x^2,...,x^L\}$ where $x^i = \{x^i_n\}_{n=1}^{dm}$  For $\lambda = \{\lambda_1,\lambda_2,...,\lambda_{L}\}\in \Delta^{L}$, we define $E(\lambda)$ by 
 \begin{align}
     E(\lambda): = \Pi_{i=1}^{dm}\left( \sum_{j=1}^L\lambda_j x^j_{i} \right).\nonumber 
 \end{align}
 Then $E(\lambda)$ is bounded and $\Delta^{L}$ is compact. Consider to maximize $\ln\left(E(\lambda)\right)$ over $\lambda \in \Delta^{L}$. It's not hard to verify that $\ln\left(E(\lambda)\right)$  is concave in $\lambda$, so it is efficient to maximize it by gradient ascent algorithms. 
 Let $\lambda^*$ be the optimal solution.
 By the KKT condition, we have that for any $j',j''$ such that $\lambda^*_{j'},\lambda^*_{j''}\in (0,1)$, it holds that
 \begin{align}
    w:= \sum_{i=1}^{dm} \frac{x_i^{j'}}{\sum_{j=1}\lambda_j x_i^j} =  \sum_{i=1}^{dm} \frac{x_i^{j''}}{\sum_{j=1}\lambda_j x_i^j}.\nonumber
 \end{align}
 Therefore, if for any $\lambda^*_j\neq 1$  for any $j$, we have that
 \begin{align}
   w=  \sum_{j=1}^L \lambda_j w = \sum_{i=1}^{dm} \frac{\sum_{j=1}\lambda_j x_i^j}{ \sum_{j=1}\lambda_j x_i^j}=dm. \nonumber
 \end{align}
 Then $\lambda^*$ is the desired solution. Otherwise, suppose $\lambda^*_1=1$. Then we have that
 \begin{align}
    dm= \sum_{i=1}^{dm} \frac{x_i^{1}}{x_i^1} \geq  \sum_{i=1}^{dm} \frac{x_i^{j'}}{x_i^1}\nonumber 
 \end{align}
 for any $j'\geq 2$. Then $\lambda^*$ is also the desired solution. The proof is completed.
\end{proof}

\subsection{Proof of Lemma~\ref{lemma:stp}}\label{app:pflemma2}
\textbf{Lemma~\ref{lemma:stp}} (Restatement) \emph{
	Let $\mathcal{P} =\otimes_{(h,s,a)} \mathcal{P}_{h,s,a}$ be a set of transition models such that $\mathcal{P}_{h,s,a}\subset \Delta^{S}$ is convex for any $(h,s,a)$. Let
	$\{(\pi^i,P^i)\}_{i=1}^n$ be a sequence of policy-transition pairs such that $P^i\in \mathcal{C}$.
	For any $\{\lambda_{i}\}_{i=1}^n$ such that $\lambda_i\geq 0$ for $i\geq 1$ and $\sum_i \lambda_i = 1$, there exists a policy $\pi$ and $P\in \mathcal{P}$, satisfying that
	\begin{align}
	W^{\pi}(\textbf{1}_{h,s,a}, P)= \sum_i \lambda_i W^{\pi^i}(\textbf{1}_{h,s,a}, P^i)\label{eq:stp11}
	\end{align}
	for any $(h,s,a)\in [H]\times \mathcal{S}\times\mathcal{A}$. Furthermore, the time complexity to find  $\{\pi,P\}$ could be bounded by $O(nS^3A^2H^2)$.
}

\begin{proof}
By induction on $n$, it suffices to prove for the case $n=2$. 
	Our target is to find $(\pi,p)$ such that
	\begin{align}
	W^{\pi}(\textbf{1}_{h,s,a},P)  = \lambda_1 W^{\pi^1}(\textbf{1}_{h,s,a}, P^{1})+(1-\lambda_1) W^{\pi^2}(\textbf{1}_{h,s,a},P^2)\label{eq:tec1}
	\end{align}
	holds for any $(h,s,a)\in [H] \times\mathcal{S}\times \mathcal{A}$. We will prove this by induction on $h$. For the case $h=1$, since the initial distribution is fixed, we finish by letting 
	\begin{align}
	\pi_{1}(a|s) = \lambda_1\pi^1_{1}(a|s)+ (1-\lambda_1)\pi^2_{1}(a|s)
	\end{align}
	for all $(s,a)\in \mathcal{S}\times \mathcal{A}$
	
	Suppose \eqref{eq:tec1} holds for any $1\leq h'\leq h$ and any $(s,a)\in\mathcal{S}\times\mathcal{A}$. Then $\lambda_{h,s,a} = \frac{\lambda_1 W^{\pi^1}(\textbf{1}_{h,s,a},P^1)  }{W^{\pi}(\textbf{1}_{h,s,a},P) }$ is well-defined. We set
	\begin{align}
	P_{h,s,a} = \lambda_{h,s,a}P^1_{h,s,a}+(1-\lambda_{h,s,a})P^2_{h,s,a}\nonumber
	\end{align}
	for any $(s,a)]\in \mathcal{S}\times\mathcal{A}$. By the inductive assumption
	\begin{align}
	W^{\pi}(\textbf{1}_{h,s,a},P)  =\lambda_1 W^{\pi^1}(\textbf{1}_{h,s,a}, P^{1})+(1-\lambda_1)W^{\pi^2}(\textbf{1}_{h,s,a}, P^{2}),
	\end{align}
	we have that for any $(s,a)$
	\begin{align}
	P_{h,s,a} W^{\pi}(\textbf{1}_{h,s,a},P)=\lambda_1 W^{\pi^1}(\textbf{1}_{h,s,a}, P^{1}) P^1_{h,s,a} +(1-\lambda_1)W^{\pi^2}(\textbf{1}_{h,s,a}, P^{2}) P^2_{h,s,a}.\nonumber
	\end{align}
	
	We then have that 
	\begin{align}
	& \lambda_1\sum_{s',a'}W^{\pi^1}(\textbf{1}_{h,s',a'},P^1)P^1_{h,s',a',s}+(1-\lambda_1)\sum_{s',a'}W^{\pi^2}(\textbf{1}_{h,s',a'},P^2)P^2_{h,s,',a',s}\nonumber
	\\ & = \sum_{s',a'}\lambda_{h,s',a'}W^{\pi}(\textbf{1}_{h,s',a'},P)P^1_{h,s',a',s}+ \sum_{s',a'}(1-\lambda_{h,s',a'})W^{\pi}(\textbf{1}_{h,s',a'},P)P^2_{h,s',a',s}
	\\ & =\sum_{s',a'}W^{\pi}(\textbf{1}_{h,s',a'},P)P_{h,s',a',s},
	\end{align}	
	which implies that
	\begin{align}
	W^{\pi}(\textbf{1}_{h+1,s},P) =  \lambda_1 W^{\pi^1}(\textbf{1}_{h+1,s}P^1) +(1-\lambda_1)W^{\pi^2}(\textbf{1}_{h+1,s}P^2)
	\end{align}
	for any $s\in \mathcal{S}$, where the reward function $\textbf{1}_{h+1,s}=\sum_{a}\textbf{1}_{h+1,s,a}$.
	Let 
	\begin{align}
	\pi_{h+1}(a|s) = \frac{\lambda_1 W^{\pi^1}(\textbf{1}_{h+1,s},P^1)\pi^1_{h+1}(a|s) + (1-\lambda_1) W^{\pi^2}(\textbf{1}_{h+1,s},P^2)\pi^2_{h+1}(a|s)  }  {W^{\pi}(\textbf{1}_{h+1,s},P)}.
	\end{align}
	Then it is easy to verify that 
	\begin{align}
	W^{\pi}(\textbf{1}_{h+1,s,a},P) = W^{\pi}(\textbf{1}_{h+1,s},P)\pi_{h+1}(a|s) = \lambda_1 W^{\pi^1}(\textbf{1}_{h+1,s,a},P^1)+ (1-\lambda_1)W^{\pi^2}(\textbf{1}_{h+1,s,a},P^2).\nonumber
	\end{align} 
	Also note that the process above costs at most $O(S^3A^2H^2)$ time, so the total computational cost is bounded by $O(nS^3A^2H^2)$. 
	The proof is completed.
\end{proof}

\end{document}